\documentclass[lettersize,journal]{IEEEtran}
\usepackage{amsmath,amsfonts}
\usepackage{array}
\usepackage[caption=false,font=normalsize,labelfont=sf,textfont=sf]{subfig}
\usepackage{textcomp}
\usepackage{stfloats}
\usepackage{url}
\usepackage{verbatim}
\usepackage{graphicx}
\usepackage[utf8]{inputenc} 
\usepackage[T1]{fontenc}    
\usepackage{hyperref}       
\usepackage{url}            
\usepackage{booktabs}       
\usepackage{nicefrac}       
\usepackage{microtype}      
\usepackage{xcolor}         
\usepackage[square,numbers,sort&compress]{natbib}
\usepackage{subfig} 
\usepackage{subfloat}
\usepackage{fancyhdr}
\usepackage{wrapfig}
\usepackage{graphicx}
\usepackage{bbm}
\usepackage{multicol}
\usepackage{booktabs}
\usepackage{multirow}
\usepackage{threeparttable}
\usepackage{algorithm}
\usepackage{algorithmicx}
\usepackage{algpseudocode}
\usepackage{amsmath} 
\usepackage{color}
\usepackage{float}
\usepackage{multicol}
\usepackage{caption}
\usepackage{booktabs}
\usepackage{comment}
\usepackage{array}

\newcommand{\orcid}[1]{\href{https://orcid.org/#1}{\includegraphics[width=10pt]{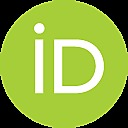}}}

\def\BibTeX{{\rm B\kern-.05em{\sc i\kern-.025em b}\kern-.08em
    T\kern-.1667em\lower.7ex\hbox{E}\kern-.125emX}}
\usepackage{balance}
\begin{document}
\title{High-order Multi-view Clustering for Generic Data}
\author{ Erlin Pan, Zhao Kang
\thanks{This work was supported in part by the
National Natural Science Foundation of China under Grant 62276053 (Corresponding author: Zhao Kang).\\
\indent Erlin Pan and Zhao Kang are with with the School of Computer Science and Engineering, University of Electronic Science and Technology of China, Chengdu, China. Email: wujisixsix6@gmail.com, zkang@uestc.edu.cn}}

\markboth{IEEE TRANSACTIONS ON NEURAL NETWORKS AND LEARNING SYSTEMS}%
{High-order Multi-view Clustering for Generic Data}

\maketitle

	\begin{abstract}
		Graph-based multi-view clustering has achieved better performance than most non-graph approaches. However, in many real-world scenarios, the graph structure of data is not given or the quality of initial graph is poor. Additionally, existing methods largely neglect the high-order neighborhood information that characterizes complex intrinsic interactions. To tackle these problems, we introduce an approach called high-order multi-view clustering (HMvC) to explore the topology structure information of generic data. Firstly, graph filtering is applied to encode structure information, which unifies the processing of attributed graph data and non-graph data in a single framework. Secondly, up to infinity-order intrinsic relationships are exploited to enrich the learned graph. Thirdly, to explore the consistent and complementary information of various views, an adaptive graph fusion mechanism is proposed to achieve a consensus graph. Comprehensive experimental results on both non-graph and attributed graph data show the superior performance of our method with respect to various state-of-the-art techniques, including some deep learning methods. 
	\end{abstract}


	\section{Introduction}
	 \IEEEPARstart{A}{s}  one fundamental task in machine learning, clustering aims to divide a collection of unlabeled objects into multiple categories. The classical clustering methods, like $k-$means and DBSCN, are heavily dependent on the first-order relationships between samples. Specifically, they employ the pre-defined distances that measure the similarity between data points to perform the clustering. However, samples are not only similar to their neighbors but also tend to be similar to the neighbors’ neighbors. Therefore, the first-order relationships are incomplete \cite{line1} and the valuable information hidden in high-order proximity should be explored.
	We take two moons data as an example and show its KNN graphs ($K=5$) in different orders $n$ in Fig. \ref{twoMoons}. The bold red edge connects two neighbor nodes that belong to different classes. The number of wrong connections (red edges) and the ratio of accurate edges over whole ones are denoted by \textit{NwE} and \textit{AccE}, respectively. It's obvious that high-order graph decreases wrong relations.\\
	 \indent To incorporate high-order interactions, numerous methods have been proposed in the past decades. They can be categorized into two classes based on the used information: topological structure-based methods 
	 and attribute-based methods 
	 . The topological structure-based methods focus on graph data and obtain high-order proximity by processing the adjacency matrix.
	 \cite{20high} is a spectral clustering method for directed graph, demonstrating that high-order structure in the adjacency relationship is an important factor for clusters.  \cite{hose3} proposes a graph clustering framework that can preserve high-order structures with the help of a novel diffusion core.  \cite{lin2021graph} regards high-order proximity as a random walk from one node to the other with different steps. A polynomial of adjacency matrices is applied to capture neighborhood relationships of various orders. However, the raw graph is often sparse and noisy, high-order information directly derived from the product of raw adjacency matrix accumulates errors, which could lead to inferior performance. Additionally, topological structure-based methods are not applicable to non-graph data.\\
	  \indent Different from topological structure-based methods, which largely neglect the high-order interactions derived from the attributes of data, the attribute-based methods can handle both graph and non-graph data. One straightforward approach to obtain the high-order similarity is constructing hypergraph. \cite{tangchang} captures the high-order correlations among samples as well as various views via a hypergraph-induced hyper-Laplacian regularization term and a tensor factorization term. \cite{ehc6} obtains the second-order or high-order affinity by a score approximated by the Taylor expansion of solution at each iteration.  \cite{hcpr7} explores the high-order relationships from local and global perspectives.  \cite{cahsm8} is a robust spectral clustering method through a context-aware hypergraph similarity measure, in which the affinity information from three types of hypergraphs is combined together to explore the intrinsic interactions. Although  \cite{largeyp9, largehp10} show that larger hyperedges are better for hypergraph clustering, hypergraph-based methods suffer from large storage and complicated computation in the construction stage.
	\begin{figure*}[t]
        	\centering
        	\subfloat[]{
        		\includegraphics[width=0.31\textwidth]{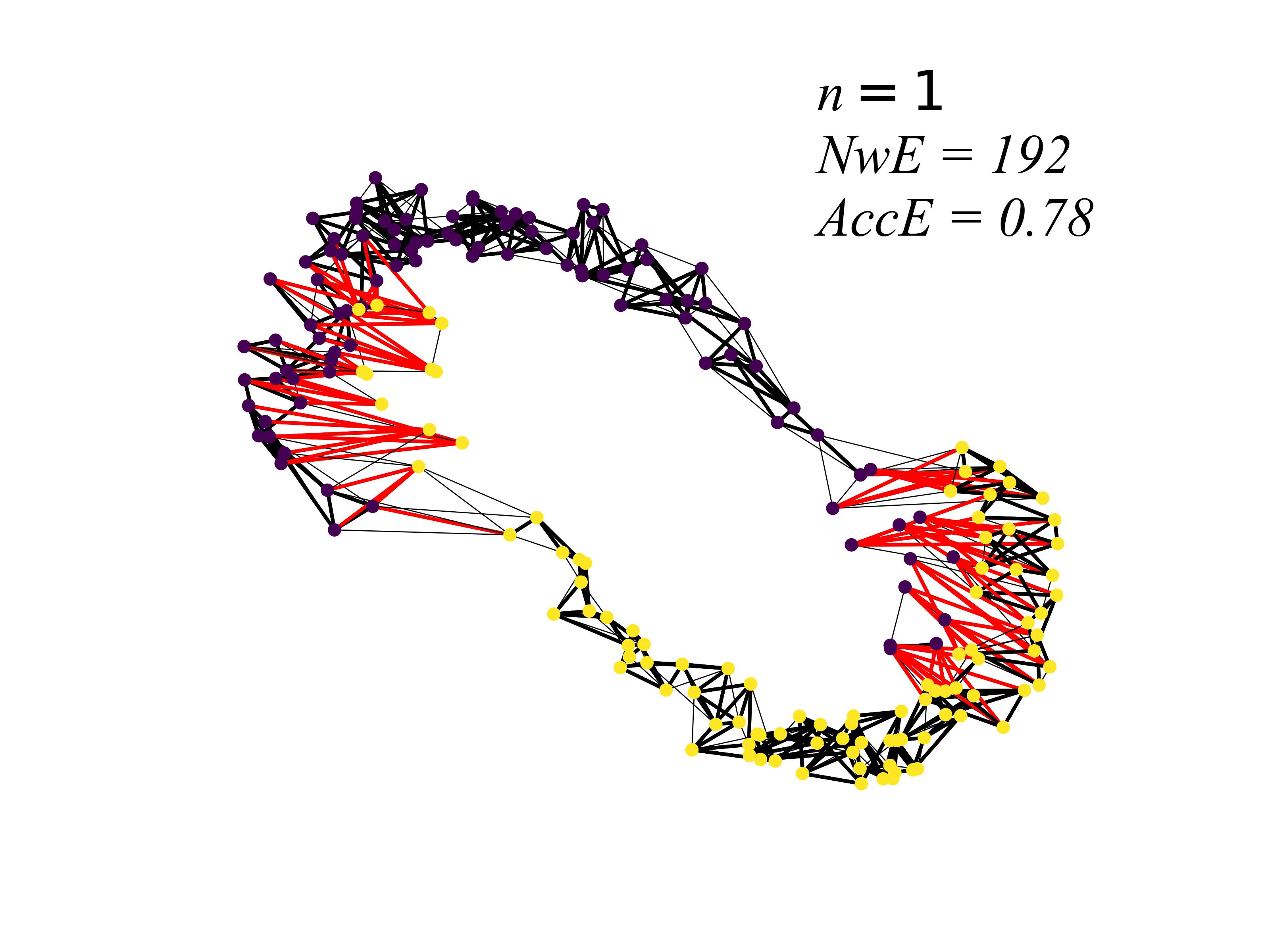}
        	}
        	\subfloat[]{
        		\includegraphics[width=0.31\textwidth]{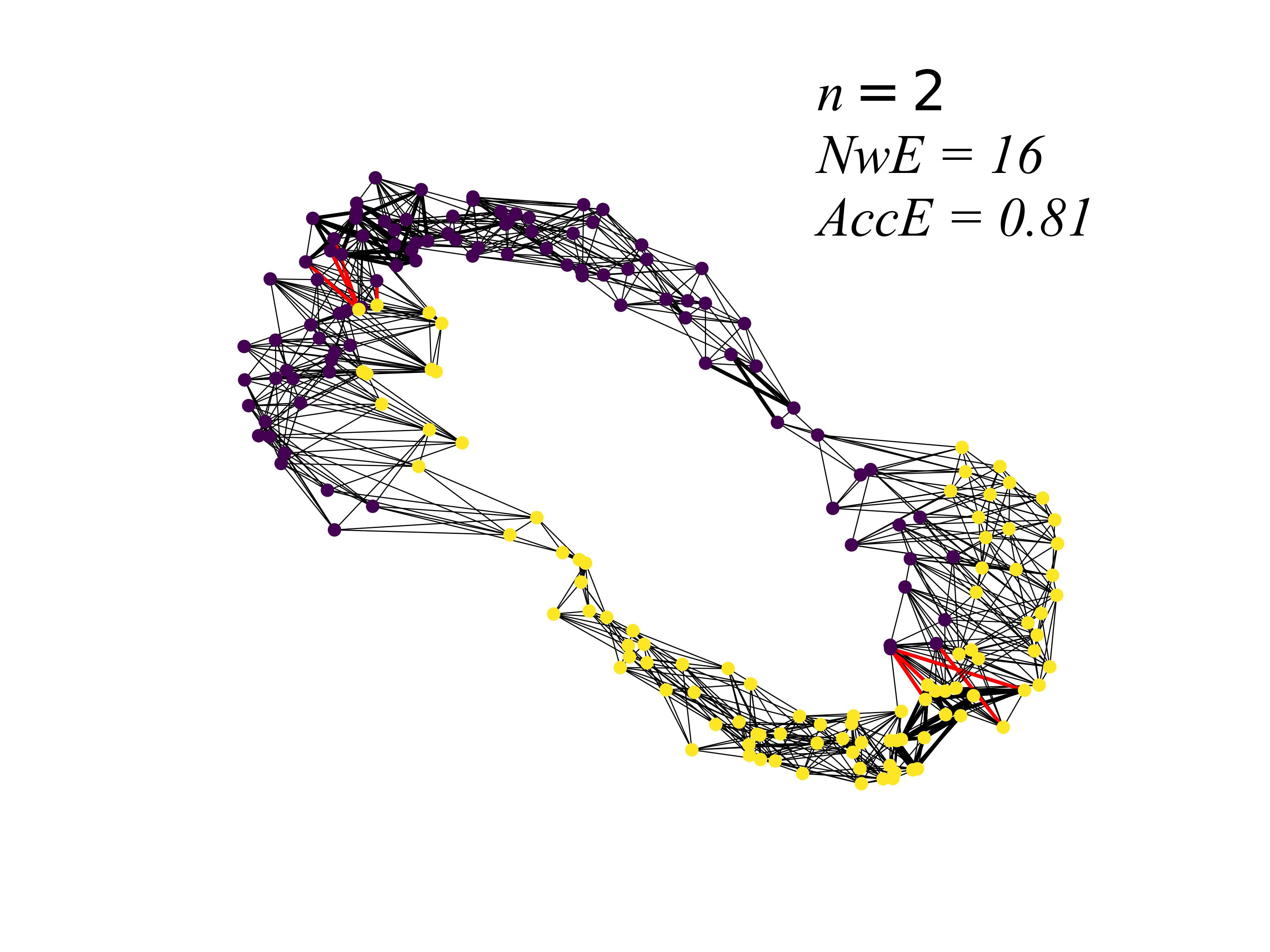}
        	}
        		\subfloat[]{
        		\includegraphics[width=0.31\textwidth]{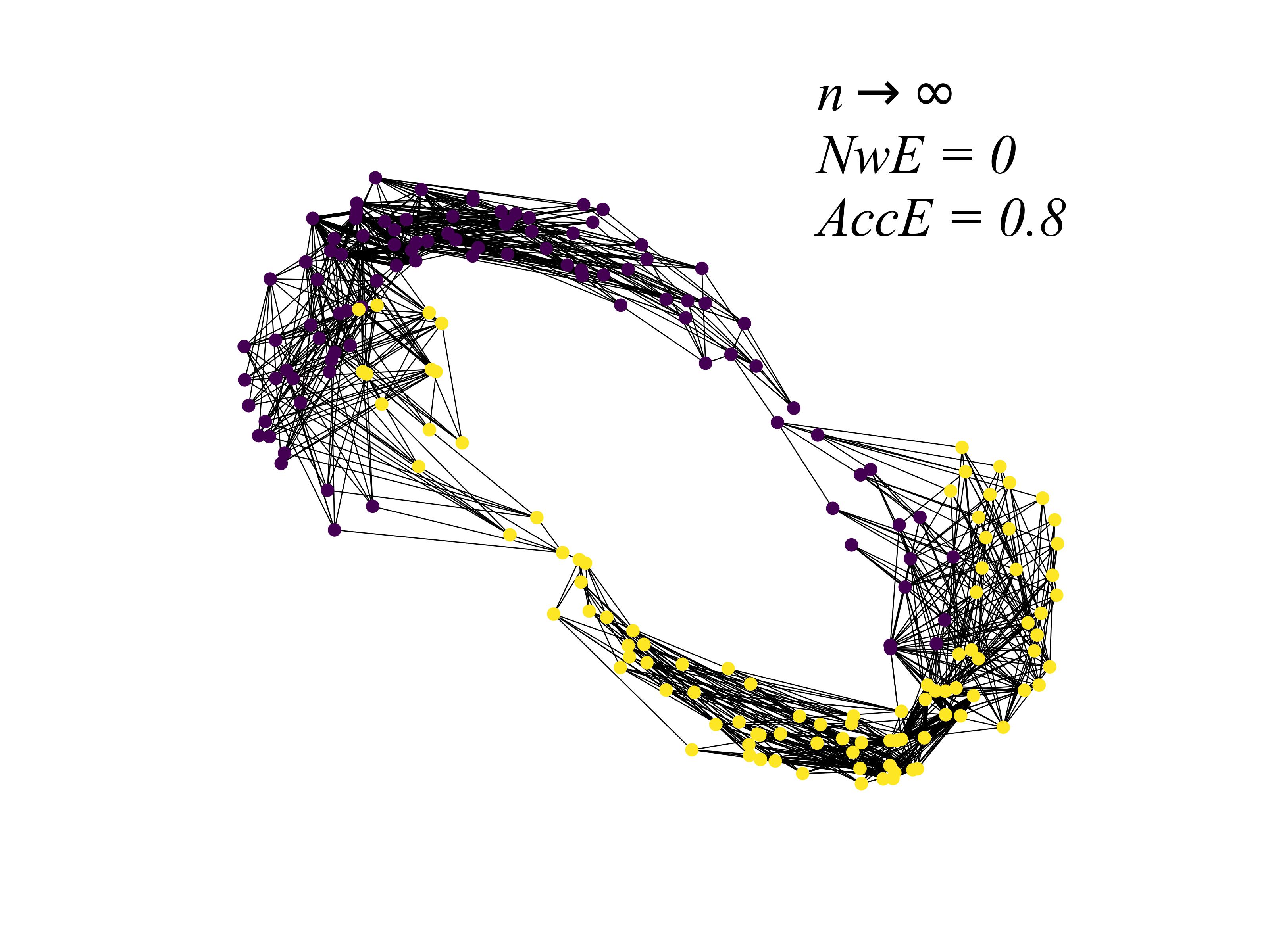}
        	}
        	\caption{The KNN graphs of two moons data. As observed, there are many wrong connections in the 1st-order graph and no errors in the infinity-order graph.}
        	\label{twoMoons}
        \end{figure*}
        
	  \indent Recently, many practical applications involve multi-view data, which provide complementary information to help improve the task performance. Different from aforementioned methods,  \cite{wmsf11} conducts multi-view clustering through self-weighted high-order similarity fusion and utilizes a cross-view strategy to get adaptive high-order similarity. 
 Nevertheless, the obtained similarity is not the intrinsic high-order information from data’s attributes. \\
	 \indent To tackle above problems, we propose High-order Multi-view Clustering (HMvC), which utilizes the powers of similarity graph to capture high-order interactions among samples. The similarity graph reflects the neighbor interactions and is derived from the attribute of data. HMvC smooths attribute through graph filtering technique and learns a consensus graph from multiple views via an adaptive graph fusion mechanism. More importantly, HMvC is directly applicable to generic data with attributed features regardless of whether they contain graph or not. We summarize our contributions as follows:\\
	
	\begin{itemize}
		\item We propose a generic  multi-view clustering method HMvC, which is suitable for both attributed graph data and non-graph data. Particularly, HMvC learns clustering-favorable representations through graph filtering and handles multi-view data through its adaptive graph fusion strategy.
		
		\item To capture high-order information, we use the powers of similarity matrices to  characterize high-order interactions. Theoretically, HMvC can obtain arbitrary-order similarity, even the infinity-order that explores all propagation paths passing similarity between samples.
		
    	\item The proposed method achieves promising performance on both multi-view non-graph and attributed graph datasets. Our simple method even outperforms many deep learning approaches.
    
    \end{itemize}

    	\section{Related Work} 
				\label{related work}
            	\subsection{Multi-view Clustering}
            	\label{mvcm}
                \indent Generally, multi-view clustering (MVC) methods focus on the global
            	consensus and complementary information carried by multiple views to improve the
            	performance \cite{huangdong, chaoguoqing}. Traditional MVC methods mainly focus on non-graph data and obtain the clustering results based on attributed features. \cite{wenjie, liuyong} focus on handle incomplete multi-view clustering via learned common representations with semantic consistency from different views. And graph-based methods, like \cite{zhankun,wangrong, wanghao}, aim to learn a consensus graph which used to obtain the clustering results from multiple attributes of data. 
            	 \cite{nie2016parameter} is a parameter-free auto-weighted multiple graph learning framework. Subspace-based methods find clusters via mapping the high-dimensional data points to several low-dimensional linear subspaces \cite{wangchandong}.  \cite{brbic2018multi} learns a joint subspace representation across all views and performs spectral clustering on the learned representation.
            	 Moreover, \cite{zhangzheng, chenyongyong} propose mul-ti-view spectral clustering based on adaptive learning mechanism.\\
            	 \indent Attributed graph data contain both attributes and topological structures, and plentiful attributed graph MVC methods have been proposed to cluster graph nodes.  \cite{mvsc26} focuses on large-scale multi-view data and bipartite
            	graph is imposed to approximate the similarity graph.  \cite{zhang2018scalable} proposes a multi-view network embedding method that learns multiple relations by a unified network embedding framework.
            	 \cite{xia2014robust} is a robust multi-view spectral clustering method and uses a shared low-rank transition probability matrix derived from each single view as input to standard Markov chain method for clustering.
            	 \cite{nie2017self} develops two methods that learn a shared graph from multiple graphs by using two weighting strategies.  \cite{fan2020one2multi} proposes multi-view graph clustering method based on graph auto-encoder and takes into account the informative view selected by modularity.  \cite{lin2021graph} is an attributed graph clustering method that exploits graph filtering and high-order neighborhood information.  \cite{chen2020graph} is a GCN-based method and mainly solves graph structured data with multi-view attributes.  \cite{lin2021completer} and  \cite{hassani2020contrastive} develop two methods based on contrastive learning: COMPLETER and MVGRL. COMPLETER performs data recovery and consistency learning of incomplete multi-view data simultaneously. MVGRL performs contrastive learning between two diffusion matrices transformed from the adjacency matrix to obtain clustering-favorable representation.\\
            	\indent We can see that different methods are developed separately to tackle graph and feature data. In fact, feature and graph can complement each other. Therefore, we develop a generic framework, which can not only be applied on non-graph but also on attributed graph data through graph filtering. More importantly, our method can explore the interaction between feature and graph.
            
            	\subsection{High-order Information Exploration}
            	\label{hoie}
            	\indent Many techniques have been developed for neighborhood information exploration ranging from the first-order to high-order relationships. For instance,  \cite{mve14} only preserves the first-order information while  \cite{line1,sdne15} exploit the second-order relationships.  \cite{grarep16} manipulates global transition matrices and combines various representations learned from multiple models, which helps to obtain a high-order information preserved graph representation.  \cite{neu17} proposes network embedding update method, in which the powered adjacency matrices are treated as high-order proximity matrices, and learns embeddings from the approximation of high-order proximity matrices. By applying the high-order Cheeger’s inequality ,  \cite{prone18} develops a scalable network embedding approach, which makes the obtained embeddings capture the high-order structural information. To take advantage of topological relationships in graph,  \cite{arope20} extracts arbitrary-order proximities with the help of eigen-decomposition reweighting theorem.  \cite{mnmf21} incorporates the community and high-order structures in an attempt to preserve as much structural information as possible. \\
            	 \indent However, above methods heavily depend on the adjacency matrix, which limits them to handle non-graph data. Recently,  \cite{hdmi22} proposes a novel framework to obtain the node embeddings containing abundant extrinsic and intrinsic information through optimization of a joint supervision signal designed by high-order mutual information, and it gets rid of dependence on the adjacency relationships. Differently, our method focuses on the construction of high-order similarity based on attributes, which reflects the intrinsic relationships in feature space. More importantly, our method explores the finite-order as well as the infinity-order information.

	\section{Methodology}
	\subsection{Notations}
	\indent Consider the non-graph data with $V$ views and $N$ data points, $X=\left\{X^1,\ldots,X^V\right\}$ are the set of feature matrices. Each $X^v={\left[x_{1}^{v}, x_{2}^{v}, \cdots, x_{N}^{v}\right]}^\top \in \mathbb{R}^{N \times d}$, where $d$ is the dimension of feature. Define the attributed graph data as $G=\left\{\mathcal{V}, E_{1}, \ldots, E_{V}, X\right\}$, where $\mathcal{V}$ is the set of $N$ nodes and $\left\{E_v\right\}^{V}_{v=1}$ represent different types of relationships, which construct various adjacency matrices $\left\{\widetilde{A}^v\right\}^{V}_{v=1} \in \mathbb{R}^{N \times N}$.
	\subsection{Graph Filtering}
	\indent A natural signal should be smooth on nearby nodes in term of the underlying graph. Graph filtering is employed to filter out undesirable high-frequency noise while retaining the graph geometric feature in graph signal processing  \cite{egilmez2018graph}. To yield clean and clustering-favorable representation, we smooth features of raw data via graph filtering technique. For the $v$-th view of data, we first add self-loop of each node on each adjacency graph, i.e., $\widetilde{A}^{v} + I$.
	The normalized adjacency matrix $A^{v}$ is computed by $A^{v}={\left({\widetilde{D}}^v\right)}^{-\frac{1}{2}}\left({\widetilde{A}}^{v}\right){\left({\widetilde{D}}^v\right)}^{-\frac{1}{2}}$, where $\widetilde{D}^v$ is corresponding degree matrix. The corresponding normalized graph laplacian $\mathrm{L}^{v}={\left({\widetilde{D}}^v\right)}^{-\frac{1}{2}}\left({\widetilde{D}}^{v}-{\widetilde{A}}^{v}\right){\left(\widetilde{D}^v\right)}^{-\frac{1}{2}}=I-A^{v}$, which can be eigen-decomposed as $L^{v}=\widetilde{U}^v \widetilde{\Lambda}^v {{\widetilde{U}}^v}{}^{\top}$, where $\widetilde{\Lambda}^v=\operatorname{diag}[\widetilde{\lambda}^v_{1},\widetilde{\lambda}^v_{2},\ldots,\widetilde{\lambda}^v_{N}]$ are the eigenvalues of $L^v$
	corresponding to eigenvectors $\widetilde{U}^v=[\widetilde{u}^v_1,\widetilde{u}^v_1,\ldots,\widetilde{u}^v_N]$. The eigenvalues are interpreted as frequencies in analogy with classical Fourier analysis. We utilize the first-order similarity introduced in \textit{section \ref{3.c}} as graph for non-graph data. \\
	 Then the smoothness $\Omega(\widetilde{u}^v_i)$ of eigenvector $\widetilde{u}^v_i$ can be measured as  \cite{cui2020adaptive}
	\begin{equation}
		\Omega(\widetilde{u}^v_i)=\frac{{\widetilde{u}^{v\top}}_{i} {L^v} {\widetilde{u}^v}_{i}}{{\widetilde{u}^{v\top}}_{i} {\widetilde{u}^v}_{i}}=\widetilde{\lambda}^v_{i}.
	\end{equation}
	 Eq. (2) indicates that a smooth graph signal contains more low-frequency (small eigenvalues) basis signals than high-frequency ones.
	  $X^v$ can be decomposed into a combination of basis signals, i.e., $X^v=\widetilde{U}^v c=\sum_{i=1}^{N}  \widetilde{u}^v_{i}c_{i}$, where $c_{i}$ is the coefficient. Moreover, the spectral response $f(\lambda)$ of a loss-pass filter $F$ can be defined as $f:\lambda \mapsto f(\lambda)=1-\tau{\lambda}$. The top eigenvalue of $L^v$ is $2$ and  \cite{cui2020adaptive} suggests that $\tau=\frac{1}{\lambda_{\operatorname{max}}}$ is the optimal choice. Then the filtered signal $H^v$ is defined as  
	\begin{equation}
		\begin{aligned}
			{H}^v&={\widetilde{U}^v F {\widetilde{U}^{v\top}}} {X^v}=\widetilde{U}^v F {\widetilde{U}^{v\top}}\widetilde{U}^v c\\&={\widetilde{U}^v}({I}-\frac{1}{2} \Lambda^v) {\widetilde{U}^{v\top}}{\widetilde{U}^v} {c}\\&=(I-\frac{1}{2}L^v)X^v
		\end{aligned}
	\end{equation}
    \indent This can only explore the first-order neighborhood information. In practice, we use the $k$-th order filtering to mine high-order relation, i.e., 
	\begin{equation}
		H^v= (I-\frac{1}{2}L^v)^kX^v.
		\label{filter1}
	\end{equation}
	
	\subsection{High-order Graph in Feature Space}\label{3.c}
    \indent Above approach is from the perspective of topological structure, here we further explore the similarity information in feature space. Similarity is a basic relationship among data points, and our model learns valuable information via making use of high-order relation. Given an attributed graph data $G$ or generic multi-view data X, we can obtain the first-order similarity of feature matrix. For the first-order similarity matrix of $v$-th view $\widetilde{W}_{v}^1 \in \mathbb{R}^{N \times N}$, its element is defined as 
	\begin{equation}
		{(\widetilde{W}_{v}^1})_{i j}=\frac{\cos <x_{i}^{v}, x_{j}^{v}>}{2}+\frac{1}{2},
	\end{equation}
	where $\cos <x_{i}^{v}, x_{j}^{v}>={x_{i}^{v}}^{\top} x_{j}^{v} /\left\|x_{i}^{v}\right\|\left\|x_{j}^{v}\right\|$ is cosine similarity between node $i$ and node $j$. If $i=j$, define $({\widetilde{W}_{v}^1})_{i j}=0$. Then normalized $W_v^1$ is defined as before. Let $ \Lambda^v = diag[\lambda_{1}^{v}$, $\lambda_{2}^{v}$, $\ldots$, $\lambda_{N}^{v}$] and $U^v = [u_{1}^{v}$, $u_{2}^{v}$, $\ldots$, $u_{N}^{v}]$ be the eigenvalues and corresponding eigenvectors of $ W_{v}^{1}$. The high-order neighborhood relationship between data points can be obtained from high-order similarity graph. We define the $n$th-order graph $W_{v}^{n}=W_{v}^{n-1} \times W_{v}^{1}$. 
	
	\newtheorem{lemma}{\textbf{Lemma}}
	\indent \\
	\begin{lemma}
		\textit{$W_{v}^n\space (n\geq1)$ is a Markov matrix, i.e., $({W_{v}^n})_{ij}\geq0$ and $\sum^{N}_{j=1}\left| ({W_{v}^n})_{ij}\right|=1$ for arbitrary $i$, $j$ in the $v$-th view.}
		\label{lemma1}
	\end{lemma}
    \newtheorem{proof}{\textbf{Proof}}
    \begin{proof}
        	\textit{Obviously, $W_v^1$ is a Markov matrix. $W_v^2=W_v^1 \times W_v^1$ and the sum of arbitrary column $({W_{v}^{2}})_{i}$ can be written as
                		$$
                		\begin{aligned}
                			\sum_{j=1}^{N} ({W_{v}^{2}})_{ i j} &=\sum_{j=1}^{N} \sum_{l=1}^{N}     ({W^{1}_{v}})_{ i l} ({W^{1}_{v}})_{ l j} \\
                			&=\sum_{l=1}^{N}\left[({W^{1}_{v}})_{ i l} \sum_{j=1}^{N} ({W^{1}_{v}})_{ l j}\right] \\
                			&=\sum_{l=1}^{N} ({W^{1}_{v}})_{ il}\\&=1.
                		\end{aligned}
                		$$
                    \indent  This means that powered $W_v^1$, like $W_v^2$, is Markov matrix too. Because $W_v^n$ is the product of $n$ $W_v^1$s, we obtain the same conclusion for $n>=3$. }
    \end{proof}
	\indent \\
	\begin{lemma}
		\textit{ For arbitrary eigenvalue $\lambda_{i}^v$  of $ W_{v}^{n}$, $\left|\lambda_{i}^v\right| \leq 1$, and at least one of all eigenvalues is equal to one.}
		\label{lemma2}
	\end{lemma}
	\begin{proof}
	    \begin{enumerate}
    			\item \textit{According to the Gershgorin circle theorem, we have the Gershgorin circle $
    				S_{i}=\left\{z \in \mathbb{R}:\left|z- ({W_{v}^{n} })_{i i}\right| \leq \sum_{j \neq i}\left|({W_{v}^{n}})_{ij}\right|\right\} $ and 
    				$\lambda_{i} \in S=\bigcup_{j=1}^{N} S_{j} \quad(i=1,2, \cdots, \mathrm{N})$. Because $ W_{v}^{n}$ is a  Markov matrix and $ W_{v}^{1}={ W_{v}^{1}}^{\top}$, then $S_{i}=\left\{z \in \mathbb{R}:|z| \leq \sum_{j \neq i}\left|({ W_{v}^{n}})_{i j}\right|=1\right\}$. Thus, $\left|\lambda_{i}\right| \leq 1$.}\\
    			\item \textit{Define $C^v =  W_{v}^{n} - I$, it's clear that $\forall i \in\{1,2, \cdots, \mathrm{N}\},\operatorname{sum}\left(\mathrm{C}_{i}^{v}\right)=\sum_{j=1}^N c_{ij}=0$, where ${C^v_i}$ is the $i$th column of $C^v$. For each row $c_i^v$ of $C^v$, we have $$
                    \begin{aligned}
                    &\left(\begin{array}{llll}
                    1 & 1 & \cdots & 1
                    \end{array}\right)\left(\begin{array}{c}
                    c_{1}^{v} \\
                    c_{2}^{v} \\
                    \vdots \\
                    c_{N}^{v}
                    \end{array}\right) =c_{1}^{v}+c_{2}^{v}+\cdots+c_{N}^{v} \\
                    &=\left(\operatorname{sum}\left(C_{2}^{v}\right), \operatorname{sum}\left(C_{2}^{v}\right), \cdots, \operatorname{sum}\left(C_{2}^{v}\right)\right)=\textbf{0}
                    \end{aligned}
                    $$ This indicates that rows of $C^v$ are linearly dependent, so $\operatorname{det}\left(\mathrm{C}^{v}\right)=0$ and $1$ is one eigenvalue of $ W_{v}^{n}$.}
    		\end{enumerate}
	\end{proof}
	\indent \\
	\begin{lemma}
		\textit{When $n \rightarrow \infty$, $W_{v}^{n}$ converges to a stable matrix and the stable matrix is represented as $\sum_{j=1}^{r} u_{j}^{v} {u_{j}^{v}}^\top$, where $r$ is the number of eigenvalues $\lambda^{v}=1$. }
		\label{stable_graph}
	\end{lemma}
	\begin{proof}
	    \textit{$ W_{v}^{1}$ is a real symmetric matrix, then $ W_{v}^{n}$ can be computed easily:$$
    			\small
    			\begin{aligned}
    				&W_{v}^{n}=\left( W_{v}^{1}\right)^{n}=U^v {(\Lambda^v)}^{n} {U^v}^{\top}= \\
    				&\left(\begin{array}{ccc}
    					u_{1}^{v} & \cdots & u_{N}^{v}
    				\end{array}\right)\left(\begin{array}{ccc}
    					\left(\lambda_{1}^{v}\right)^{n} & \cdots & 0 \\
    					\vdots & \ddots & \vdots \\
    					0 & \cdots & \left(\lambda_{N}^{v}\right)^{n}
    				\end{array}\right)\left(\begin{array}{c}
    					u_{1}^{v \top} \\
    					\vdots \\
    					u_{N}^{v \top}
    				\end{array}\right)
    			\end{aligned}.$$ Suppose that there are $r$ eigenvalues equal to 1, then $W_{v}^{n} = \sum_{j=1}^{r} u_{j}^{v} {u_{j}^{v}}^\top$ when $n \rightarrow \infty$.}
	\end{proof}
	\indent \\
	 \indent \textit{Lemma \ref{lemma2} and \ref{stable_graph}} indicates that there exists a convergent non-zero similarity graph when $n \rightarrow \infty$ and we obtain the infinity-order graph as $\lim\limits_{n \rightarrow \infty} W_{v}^{n}= \sum_{j=1}^{r} u_{j}^{v} {u_{j}^{v}}^\top$.
	To explore the information of high-order neighborhood relationship, we define the mixed similarity graph as 
	\begin{equation}
		f_v(W)= W_{v}^{1}+W_{v}^{2}+\cdots+W_{v}^{n}.
		\label{high1}
	\end{equation} 
	\indent Specifically, the first-order proximity is the most intuitionistic relationship among nodes and the infinity-order interactions are stable relationships after moving with infinite steps. We hypothesize that the first-order and infinity-order similarity graph contain most information, which is later verified by our ablation study. Then we define the mixed similarity graph with the least cost when $n \rightarrow \infty$ as 
	\begin{equation}
		f_v(W)= W_{v}^{1}+ \sum_{j=1}^{r} u_{j}^{v} {u_{j}^{v}}^\top.
		\label{high2}
	\end{equation}
	\subsection{Graph Learning}
	\indent Since the raw graph is often sparse and noisy, which is not optimal for downstream tasks, we choose to learn a refined graph from data. For multi-view data, a graph can be learned for each view based on self-expression property of data, i.e., each data point can be represented by a linear combination of other samples. The objective function on single-view data can be mathematically formulated as
	\begin{equation}
		\min _{S^v}\left\|{H^v}^{\top}-{H^v}^{\top} S^v\right\|_{F}^{2}+\alpha \|S^v\|_F^2,
		\label{single}
	\end{equation}
	where $\alpha>0$ is a trade-off parameter and the Frobenius norm is used as a regularizer for simplicity. To incorporate high-order neighborhood relationship, we require $S^v$ be close to the mixed similarity graph. Though some common information is shared by all views, there exists some complementary factors, such as geometry and semantics, which are distinct in different views. To obtain a consistent result, we propose a graph fusion framework, which assigns weights to various views automatically. Eventually, our proposed HMvC model can be written as 
	\begin{equation}
		\begin{aligned}
		\centering
			\min _{S^{v}, \gamma^{v}, S} &\sum_{v=1}^{V} \gamma^{v} \left(\left\|H^{v \top}-H^{v \top} S^{v}\right\|_{F}^{2} +\alpha\|S^v-f_v(W)\|_{F}^{2}\right)\\ &+\beta\left\|S-\sum_{v=1}^{V} \gamma^{v} S^{v}\right\|_{F}^{2}+\mu\|S\|_{F}^{2}, \\ 
			&\text { s.t. } \sum_{v=1}^{V} \gamma^{v}=1, \quad 0 \leq \gamma^{v},
			\centering
		\end{aligned}
		\label{obj}
	\end{equation}
	where the third term is the graph fusion term that aims to obtain a unified graph $S$ by assuming it is a linear combination of $S^v$ from each view, $\gamma^{v}$ is the weight of $v$-th view. If the value in the parentheses of first term is big, the corresponding $\gamma^v$ would be small, which contributes little to final $S$. It can be seen that there are three key components in HMvC: smoothing features via graph filtering, preserving high-order similarity information, and learning a unified graph through graph fusion mechanism. Note that graph filtering explores the high-order relation in topology structure, while high-order graph mines the high-order information in feature space. They complement each other to some extent.
		\begin{equation}
			\begin{aligned}
				\min _{S^{v}, \gamma^{v}, U} &\sum_{v=1}^{V} \gamma^{v}\left(\left\|H^{v \top}-H^{v \top} S^{v}\right\|_{F}^{2}+\alpha\left\|S^{v}-f(W)\right\|_{F}^{2}\right) \\ &+\beta\left\|U-\sum_{v}^{V} \gamma^{v} S^{v}\right\|_{F}^{2}+\mu\|U\|_{F}^{2} \\
				&\text { s.t. } \quad \sum_{v}^{V} \gamma^{v}=1, \quad 0 \leq \gamma^{v},
			\end{aligned}
		\end{equation}
	
	\subsection{Optimization Procedure}
	\indent To solve problem (\ref{obj}), we adopt an alternating optimization strategy, in which we fix one variable at once and optimize the others alternatively.\\
	\textbf{1) Initialization of $S$}: Firstly, we initialize $S$ as the view average of solutions from Eq. (\ref{single}). \\
	\textbf{2) Fix $S$ and $\gamma^{v}$, update $S^v$}: When $S$ and $\gamma^{v}$ are fixed, we rewrite Eq. (\ref{obj}) for each view as follows
	\begin{equation}
		\begin{aligned}
			\min _{S^{v}} &\left\|{H^{v}}^{\top}-{H^{v}}^{\top} S^{v}\right\|_{F}^{2}+\alpha\left\|S^{v}-f_v(W)\right\|_{F}^{2}+\\& \beta\left\|\left(S-\sum_{i\neq v}^V \gamma^i S^i\right)-\gamma^{v} S^{v}\right\|_{F}^{2}.
		\end{aligned}
	\end{equation}
	By setting its derivation w.r.t $S^v$ to zero, we get
	\begin{equation}
		\begin{aligned}
			S^{v}= &\left(H^{v} {H^{v}}^{\top}+\left(\alpha+\beta \gamma^{v}\right) \mathrm{I}\right)^{-1}(H^{v}{H^{v}}^{\top}+\alpha f_v(\mathrm{~W})\\&+\beta (S-\sum_{i\neq v}^V \gamma^i S^i)).
		\end{aligned}
		\label{UP_S_v}
	\end{equation}
	\textbf{3) Fix $S^v$ and $\gamma^{v}$, update $S$}: When $S^v$ and $\gamma^{v}$ are fixed, Eq. (\ref{obj}) is simplified into
	\begin{equation}
		\min _{S}\beta\left\|S-\sum_{v=1}^{V} \gamma^{v} S^{v}\right\|_{F}^{2}+\mu\|S\|_{F}^{2}.
	\end{equation}
	Let its derivation w.r.t $S$ be zero, we have
	\begin{equation}
		S=\frac{\beta \sum_{v=1}^{V} \gamma^{v} S^{v}}{\beta+\mu}.
		\label{UP_S}
	\end{equation}
	
    \textbf{4) Fix $S$ and $S^{v}$, update $\gamma^v$}: Define ${\gamma}=\left(\gamma^{1}, \gamma^{2}, \cdots, \gamma^{V}\right)$, then our objective function is formulated as a quadratic programming (QP) problem
	
	\begin{equation}
		\begin{array}{cc}
			\underset{\gamma}{\operatorname{min}} \quad \frac{1}{2} {\gamma}^{\top} P {\gamma}+q^{\top} {\gamma}, \text { s.t.} \quad {\gamma}>\textbf{0}, \textbf{1} \cdot {\gamma}=1,
		\end{array}
		\label{qp}
	\end{equation}
	 
	where $P \in \mathbb{R}^{V \times V}$ and $P_{i j}=\operatorname{Tr}\left(S^{i} \times S^{j}\right)$, $q_{v}=M^{v}-2 * \beta T r\left(\mathrm{SS}^{v}\right)$ and $S^{v}=\left\|{H^{v}}^\top-{H^{v}}^\top S^{v}\right\|_{F}^{2}+\alpha\left\|S^{v}-f_v(W)\right\|_{F}^{2}$. As a standard QP problem, Eq. (\ref{qp}) can be solved efficiently. \\

	 \begin{algorithm}
	\label{Alg1}
	\begin{algorithmic}[1]
	\caption{HMvC}
		\label{al1}
		\Require 
		parameter $\alpha$, $\beta$ and $\mu$, adjacency matrix $ \widetilde{A}^{1}$,...,$\widetilde{A}^{V}$ ($W_{1}^{1}$,..., $W_{V}^{1}$ for non-graph data), feature $X^1$,...,$X^V$, the order of graph filtering $k$, the number of clusters $c$
		\Ensure $c$ clusters
		 \State Initialize $S$ as Eq. (\ref{single}) and set $\gamma^{v}=\frac{1}{V}$;
		 \State Graph filtering by Eq. (\ref{filter1}) for each view;
		 \State Compute mixed high-order similarity graph $f_v(W)$ via Eq. (\ref{high1}) or Eq. (\ref{high2});
		\While{convergence condition does not meet}:
		 \State Update $S^v$ in Eq. (\ref{UP_S_v}) for each view;
		 \State Update $S$ in Eq. (\ref{UP_S});
		 \State Update $\gamma^{v}$ via solving Eq. (\ref{qp});
		\EndWhile
		\State Clustering on $S$.
	\end{algorithmic}
    \end{algorithm} 
	

	 \indent The optimization procedure will monotonically decrease the objective function value in Eq. (\ref{obj}) in each iteration  \cite{bezdek2003convergence}. Since the objective function has a lower bound, such as zero, the above iteration converges. The complete steps of HMvC are outlined in Algorithm \ref{al1}. Afterwards, we perform classical clustering method on obtained $S$, like $k$-means and spectral clustering, etc., to achieve the final result.
	\subsection{Scalable Graph Learning for Large-scale Data}
	\indent Although Eq. (\ref{obj}) can construct a clustering-favorable graph, it could take a lot of memory and time because the size of learned graph is $N\times N$, and Eq. (\ref{UP_S_v}) involves cubic computation complexity. Therefore, scalable solution is desired for large-scale data, like matrix factorization \cite{fanjicong, defast}. One efficient approach is to select some representative samples as anchors to approximate raw data. Anchor-based strategy is mainly composed of two steps:  anchor selection and the construction of anchor graph. Two popular means of anchor selection are random sampling  \cite{wang2021spatial} and k-means  \cite{kang2020large, wang2016scalable, ou2020anchor}. There are two strategies for k-means based anchor selection.  \cite{kang2020large} regards the centroid of each cluster as the landmark representation while  \cite{wang2016scalable, ou2020anchor} choose the data points close to centroids as anchors. Differently,  \cite{li2020multi} selects anchors via a weight mechanism on the raw features of data.\\
	 \indent However, all aforementioned anchor selection methods focus on attributed features only and neglect the topological information. Similar to  \cite{lin2021graph}, we choose anchors via the node importance in topological structure. Specifically, the set of anchor nodes $Y$ is chosen  from the top $m$ nodes with large degree. Different from  \cite{lin2021graph}, we consider the mixed second-order adjacency matrix $\widehat{A^{v}}=\widetilde{A^{v}}+\widetilde{A^{v}}^{2}$. The multi-view degree matrix is computed as $\widehat{D}=\sum_{v=1}^V \widehat{D}^v$, where $\widehat{D}^v$ is the mixed degree matrix of $v$-th view.
	Then the absolute importance of node $i$ is $g_{i}=\frac{\left(\widehat{D}_{i i}\right)^{\eta}}{\sum_{j \in \mathcal{V}-Y}\left(\widehat{D}_{j j}\right)^{\eta}}$, where $\eta>1$ is a sharpening parameter. We repeatedly select anchors via the relative importance of node
	\begin{equation}
		p_{i}=\frac{g_{i}}{\sum_{j \in \mathcal{V}-Y} g_{j}}.
		\label{anchor_get}
	\end{equation}
	\indent At the beginning, $Y=\emptyset$. Then we iteratively pick node $i$ with the largest $p$ in $\mathcal{V}-Y$. After each iteration, $Y=Y\cup \left\{i\right\}$. The indexes of selected anchors are represented as $Inds =\left[S_{1}, \mathrm{~s}_{2}, \cdots, S_{m-1}, \mathrm{~s}_{m}\right]$, then the filtered anchor representation $\widetilde{H}^v $ are picked from ${H^v}$ by index $Inds$. The anchor-based  mixed high-order graph $\widetilde{f_v(W)}$ is constructed by Eq. (\ref{high1}) and \textit{Lemma \ref{anchor_lemma}} efficiently.
	\indent \\
	\begin{lemma}
		Divide the $n$th-order similarity graph of anchors $W_{v}^{n}$ into two parts:
		$W_{v}^{n}=\left(\begin{array}{c}(W_{v}^{n})^1\\ (W_{v}^{n})^2\end{array}\right)$, $(W_{v}^{n})^1=\left(\begin{array}{ll}(W_{v}^{n-1})^1 {(W_{v}^{1})^1}^{\top} & (W_{v}^{n-1})^1  {(W_{v}^{1})^2}^{\top}\end{array}\right)$ and $(W_{v}^{n})^2$ is the remaining part, where $ {(W_{v}^{1})^1}=W_{v}^{\mathrm{1}}[$ Inds $]$ and  $ {(W_{v}^{1})^2}$ is defined similarly.
		\label{anchor_lemma}
	\end{lemma}
	\begin{proof}
	    \textit{The second-order similarity graph is computed by $W^{2}_{v}=\left(\begin{array}{l} {(W_{v}^{1})^1} \\  {(W_{v}^{1})^2}\end{array}\right)\left(\begin{array}{ll}{ {(W_{v}^{1})^1}}^{\top}& { {(W_{v}^{1})^2}}^{\top}\end{array}\right)$, and  the second-order anchor-based graph is $ {(W_{v}^{2})^1}=\left(\begin{array}{ll} {(W_{v}^{1})^1} { {(W_{v}^{1})^1}}^{\top} &  {(W_{v}^{1})^1} { {(W_{v}^{1})^2}}^{\top}\end{array}\right)$. The rest can be defined in the same manner\\
    		$$\left\{\begin{array}{cc}
    			{(W_{v}^{3})^1}=\left(\begin{array}{ll} {(W_{v}^{2})^1} { {(W_{v}^{1})^1}}^{\top} &  {(W_{v}^{2})^1} { {(W_{v}^{1})^2}}^{\top}\end{array}\right); \\
    			\vdots \\
    			{(W_{v}^{n})^1}=\left(\begin{array}{ll} {(W_{v}^{n-1})^1} { {(W_{v}^{1})^1}}^{\top} &  {(W_{v}^{n-1})^1} { {(W_{v}^{1})^2}}^{\top}\end{array}\right).
    		\end{array}\right.
    		$$}
	\end{proof}
	
	The objective of anchor-based HMvC (AHMvC) is formulated as
	\begin{equation}
		\begin{aligned}
			\min _{Z^{v}, \gamma^{v}, Z} &\sum_{v=1}^{V} \gamma^{v}\left(\left\|H^{v \top}-\widetilde{H}^{v \top} Z^{v}\right\|_{F}^{2}+\alpha\left\|Z^{v}-\widetilde{f_{v}(W)}\right\|_{F}^{2}\right)\\& + \beta\left\|Z-\sum_{v=1}^{V} \gamma^{v} Z^{v}\right\|_{F}^{2}+\mu\|Z\|_{F}^{2}, \\
			&\text { s.t. } \quad \sum_{v=1}^{V} \gamma^{v}=1, \quad 0 \leq \gamma^{v}.
		\end{aligned}
	\end{equation}
    $Z$ or $Z^v$ has a size of $N\times m$, which represents the similarity between samples and anchors. The optimization strategy of AHMvC is similar to HMvC, and we show it in Algorithm \ref{al2}. As a result, the computation complexity is linear to $N$.
	\begin{algorithm}[htbp]
	\caption{AHMvC}
	\begin{algorithmic}[2]
		\Require
		parameter $\alpha$, $\beta$ and $\mu$, adjacency matrix $ \widetilde{A}^{1}$,...,$\widetilde{A}^{V}$ ($W_{1}^{1}$,..., $W_{V}^{1}$ for non-graph data), feature $X^1$,...,$X^V$, the order of graph filtering $k$, the number of clusters $c$, anchors' num $m$
		\Ensure $c$ clusters
		 \State Initialize $S$ according to Eq. (\ref{single}) and $\gamma^{v}=\frac{1}{V}$;
		 \State Graph filtering by Eq.  (\ref{filter1});
		 \State Obtain $Inds$ of anchors by Eq. (\ref{anchor_get});
		 \State Select filtered anchors' representations $\widetilde{H}^v=H^v[inds]$;
		 \State Construct anchor-based high-order similarity graphs via Lemma \ref{anchor_lemma};
		 \State Continue the remaining steps of HMVC.
	\end{algorithmic}
	\label{al2}
\end{algorithm} 
	\section{Experiment}
	\subsection{Datasets and Metrics} \label{datastss}
	\indent We evaluate HMvC on both multi-view attributed graph datasets and generic feature datasets. For multi-view attributed graph datasets, ACM, DBLP, IMDB  \cite{fan2020one2multi}, AIDS  \cite{nrr}, and Amazon datasets  \cite{shchur2018pitfalls} are chosen.\\
    	    \indent ACM is a paper network from ACM database. Its node features are the elements of a bag-of-words representing of each paper's keywords. Two graphs are constructed by two types of relationships: "Co-Author" means that two papers are written by the same author and "Co-Subject" suggests that they focus on the same field.\\
    	    \indent DBLP is an author network from DBLP database. Its node features are the elements of a bag-of-words representing of each author's keywords. Three graphs are derived from the relationships: "Co-Author", "Co-Conference", and "Co-Term", which indicate that two authors have worked together on the same paper, published papers at the same conference, and published papers with the
            same terms. \\
            \indent IMDB is a movie network from IMDB database.  Its node features correspond to elements of a bag-of-words representing of each movie. The relationships of being acted by the same actor (Co-actor) and directed by the same director (Co-director) are exploited to construct two graphs. \\
            \indent Amazon datasets contain Amazon photos and Amazon computers, which are segments of the Amazon co-purchase network. Their nodes represent goods and features of each good are bag-of-words of product reviews, the edge means that two goods are purchased together. To have multi-view attributes, the second feature matrix is constructed via cartesian product by following \cite{cheng2020multi}. \\
            \indent AIDS consists of graphs representing molecular compounds. Nodes are labeled with the number of corresponding chemical symbol and edges are constructed by the valence of linkage. We think the topological information is more important for AIDS than attributes because feature has only 4 dimensions. All statistical information of datasets is shown in Table \ref{data_1}. \\
            \indent Caltech20 and Caltech7 are frequently used subsets of Caltech101 consisting of 20/7 categories of images built for object recognition tasks. Citeseer is a citation network, whose nodes represent publications. Statistical information of them is shown in Table \ref{data_2}.\\
            
        	\begin{table}[t]
        		\centering
        	   
        		\footnotesize
        		\caption{ multi-view attributed graph datasets}
        		\setlength{\tabcolsep}{0.8mm}{
        		\begin{tabular}{ccccc}
        			\toprule
    			Dataset & Nodes & Features & Graph and Edges & Clusters \\
    			\midrule
    			\multirow{2}[4]{*}{ACM} & \multirow{2}[4]{*}{3,025} & \multirow{2}[4]{*}{1,830} & Co-Subject (29,281) & \multirow{2}[4]{*}{3} \\
    			\cmidrule{4-4}          &       &       & Co-Author (2,210,761) &  \\
    			\midrule
    			\multirow{3}[6]{*}{DBLP} & \multirow{3}[6]{*}{4,057} & \multirow{3}[6]{*}{334} & Co-Author (11,113) & \multirow{3}[6]{*}{4} \\
    			\cmidrule{4-4}          &       &       & Co-Paper (5,000,495) &  \\
    			\cmidrule{4-4}          &       &       & Co-Term (6,776,335) &  \\
    			\midrule
    			\multirow{2}[4]{*}{IMDB} & \multirow{2}[4]{*}{4,780} & \multirow{2}[4]{*}{1,232} &  Co-Actor (98,010) & \multirow{2}[4]{*}{3} \\
    			\cmidrule{4-4}          &       &       & Co-Director (21,018) &  \\
    			\midrule
    			\multirow{2}[4]{*}{Amazon photos} & \multirow{2}[4]{*}{7,487} & 745   & \multirow{2}[4]{*}{Co-Purchase(119,043)} & \multirow{2}[4]{*}{8} \\
    			\cmidrule{3-3}          &       & 7,487 &       &  \\
    			\midrule
    			\multirow{2}[4]{*}{Amazon computers} & \multirow{2}[4]{*}{13,381} & 767   & \multirow{2}[4]{*}{Co-Purchase(245,778)} & \multirow{2}[4]{*}{10} \\
    			\cmidrule{3-3}          &       & 13,381 &       &  \\
    			\midrule
    			\multirow{3}[6]{*}{AIDS} & \multirow{3}[6]{*}{25,163} & \multirow{3}[6]{*}{4} & relationships-1 (18,844) & \multirow{3}[6]{*}{37} \\
    			\cmidrule{4-4}          &       &       & relationships-2 (6,626) &  \\
    			\cmidrule{4-4}          &       &       & relationships-3 (135) &  \\
    			\bottomrule
    		\end{tabular}}
    		\label{data_1}
    	\end{table}%

    	\begin{table}[t]
    		\centering
    		\footnotesize
    		\caption{generic multi-view feature datasets}
    		\begin{tabular}{ccc}
    			\toprule
    			View  & Caltech20/Caltech7 & Citeseer \\
    			\midrule
    			1st   & Gabor(48) & Citation Links (3,312) \\
    			2nd   & Wavelet moments (40) & Words Presence (3,703) \\
    			3rd   & CENTRIST (254) &  \\
    			4th   & HOG (1,984) &  \\
    			5th   & GIST (512) &  \\
    			6th   & LBP (928) &  \\
    			\midrule
    			Nodes & 2,386/1,474 & 3,312 \\
    			\midrule
    			Clusters & 20/7 & 6 \\
    			\bottomrule
    		\end{tabular}%
    		\label{data_2}
    	\end{table}%
	\indent We adopt five popular clustering metrics, including ACCuracy (ACC), Normalized Mutual Information (NMI), F1 score,  Adjusted Rand Index (ARI), and PURity (PUR). A higher value of these metrics indicates a better clustering performance.

		\subsection{Experimental Setup}	\label{setting}
		\indent We compare HMvC and AHMvC with several state-of-the-art multi-view graph clustering methods on graph data: graph learning-based methods, like RMSC  \cite{xia2014robust}, MNE  \cite{zhang2018scalable}, PwSC  \cite{nie2017self}, and MvAGC  \cite{lin2021graph}; GNN-based methods, like O2MA, O2MAC  \cite{fan2020one2multi}, SDCN  \cite{bo2020structural}, HDMI  \cite{hdmi22}, DAEGA  \cite{wang2019attributed}, CMGEC  \cite{wang2021consistent}, and MAGCN  \cite{chen2020graph}; contrastive learning-based methods, like COMPLETER  \cite{lin2021completer} and MVGRL \cite{hassani2020contrastive}.\\
		\indent RMSC employs low-rank and sparse decomposition to achieve robust spectral clustering. MNE learns multi-view graph embedding for clustering and PwSC is a parameter-weighted multi-view graph clustering method. MvAGC is a scalable method that explores high-order interaction among topological structures. O2MA and O2MAC select the most informative view to learn representation for clustering. HDMI learns node embeddings by exploring high-order mutual information. SDCN and DAEGA are two deep clustering methods aiming to collect structural information via GCN and graph attention auto-encoder, respectively. CMGEC adds graph fusion network on multiple graph auto-encoders to obtain a consistent embedding. MAGCN applies graph auto-encoder on both attributes and topological graphs. COMPLETER and MVGRL learn a common representation shared across multiple views and different graphs via contrastive mechanism, respectively. \\
		\indent Furthermore, we compare HMvC with several traditional methods on generic multi-view data, including AMSL  \cite{nie2016parameter}, MLRSSC  \cite{brbic2018multi}, MSC\_IAS  \cite{wang2019multi}, LMVSC  \cite{kang2020large}, and SMC \cite{liu2022scalable}. AMSL is a parameter-free auto-weighted multiple
		graph learning framework and remaining methods are subspace clustering methods.\\
		\indent We set $\alpha$, $\beta$, and $\mu$ in the same range $[1e-3, 1, 1e2, 1e3, 1e4]$. On all datasets, we set filtering order $k=2$. We utilize third-order similarity on graph data while second-order on non-graph data, then we tune the parameters to obtain the best results. On the medium-size datasets, we choose the number of anchors $m=[50, 100]$ for AHMvC.
		Almost all experiments are conducted on the same machine with Intel(R) Core(TM) i7-8700 3.20GHz CPU, two GeForce GTX 1080 Ti GPUs, and 64GB RAM. MVGRL  \cite{hassani2020contrastive} is implemented on Google Colab with Intel(R) Xeon(R) CPU 2.30GHz CPU, one Tesla K80, and 12GB RAM.
		
		\subsection{Experimental Results}
		\subsubsection{Clustering task on graph data}
		The results on graph data are shown in Table \ref{re_1} and Table \ref{re_2}. Our method achieves the best performance in most cases.\\
		\begin{table*}[htbp]
		\centering
        \small
		\caption{The results of HMvC ($k=2,n=3$) and other methods. ‘OM’ means that the method raises out-of-memory problem. \textbf{Bold} numbers indicate the best score.}	
		\setlength{\tabcolsep}{0.8mm}{
		\begin{tabular}{ccccccccccccc}
            \toprule
            \multicolumn{1}{c}{\multirow{2}[4]{*}{  Method}} & \multicolumn{4}{c}{ACM}       & \multicolumn{4}{c}{DBLP}      & \multicolumn{4}{c}{IMDB} \\
        \cmidrule{2-13}          & ACC   & NMI   & ARI   & F1    & ACC   & NMI   & ARI   & F1    & ACC   & NMI   & ARI   & F1 \\
            \midrule
            MNE   & 0.6370  & 0.2999  & 0.2486  & 0.6479  & OM    & OM    & OM    & OM    & 0.3958  & 0.0017  & 0.0008  & 0.3316  \\
            RMSC  & 0.6315  & 0.3973  & 0.3312  & 0.5746  & 0.8994  & 0.7111  & 0.7647  & 0.8248  & 0.2702  & 0.0054  & 0.0018  & 0.3775  \\
            PwMC  & 0.4162  & 0.0332  & 0.0395  & 0.3783  & 0.3253  & 0.0190  & 0.0159  & 0.2808  & 0.2453  & 0.0023  & 0.0017  & 0.3164  \\
            SDCN  & 0.8631  & 0.5783  & 0.6387  & 0.8619  & 0.6497  & 0.2977  & 0.3099  & 0.6377  & 0.4047  & 0.0099  & 0.0109  & 0.3535  \\
            DAEGC  & 0.8909  & 0.6430  & 0.7046  & 0.8906  & 0.8733  & 0.6742  & 0.7014  & 0.8617  & 0.3683  & 0.0055  & 0.0039  & 0.3560  \\
            O2MAC  & 0.9042  & 0.6923  & 0.7394  & 0.9053  & 0.9074  & 0.7287  & 0.7780  & 0.9013  & 0.4502  & 0.0421  & 0.0564  & 0.1459  \\
            HDMI  & 0.8737  & 0.6453  & 0.6736  & 0.8720  & 0.8846  & 0.6918  & 0.7530  & 0.8652  & 0.5835  & \textbf{0.1692 } & \textbf{0.2033 } & 0.5003  \\
            CMGEC & 0.9089  & 0.6912  & 0.7232  & 0.9072  & 0.9103  & 0.7237  & 0.7859  & 0.9042  & 0.4844  & 0.0514  & 0.0469  & \textbf{0.5101 } \\
            MvAGC  & 0.8975  & 0.6735  & 0.7212  & 0.8986  & 0.9277  & 0.7727  & 0.8276  & 0.9225  & 0.5633  & 0.0371  & 0.0940  & 0.3783  \\
            AHMvC & 0.8981  & 0.6745  & 0.7201  &0.8992  & 0.9307  & 0.7750  & 0.8328  & 0.9263  & 0.5525  & 0.0180  & 0.0434  & 0.3188  \\
            HMvC  & \textbf{0.9110 } & \textbf{0.6988 } & \textbf{0.7521 } & \textbf{0.9124 } & \textbf{0.9322 } & \textbf{0.7844 } & \textbf{0.8370 } & \textbf{0.9276 } & \textbf{0.6006 } & 0.1001  & 0.1805  & 0.4599  \\
            \bottomrule
            \end{tabular}}
		\label{re_1}
	\end{table*}%
	
    	 
	            \begin{table*}[htbp]
        	   \renewcommand\arraystretch{0.8}
                 \small
        		\centering    
        		\caption{The results on medium-size datasets ($k=2,n=3$). Some methods are removed due to OM.}
    	       
            
        		\setlength{\tabcolsep}{0.9mm}{
        		\begin{tabular}{ccccccc}
                \toprule
                Datasets & Method & ACC   & NMI   & ARI   & F1    & Time(s) \\
            \cmidrule{2-7}    \multirow{5}[2]{*}{Amazon Photos} & COMPLETER & 0.3678  & 0.2606  & 0.0759  & 0.3067  & 421.3 \\
                      & MVGRL & 0.5054  & 0.4331  & 0.2379  & 0.4599  & 994.2 \\
                      & MAGCN & 0.5167  & 0.3897  & 0.2401  & 0.4736  & 3783.6 \\
                     & MvAGC & 0.6775  & 0.5237  & 0.3968  & 0.6397  & \textbf{37.8} \\
                      & AHMvC  & \textbf{0.7332 } & \textbf{0.6504 } & \textbf{0.5788 } & \textbf{0.7188 } & 60.5 \\
                \midrule
                \multirow{5}[2]{*}{Amazon Computers} & COMPLETER & 0.2417  & 0.1562  & 0.0536  & 0.1601  & 844 \\
                      & MVGRL & 0.2450  & 0.1012  & 0.0553  & 0.1706  & 1520 \\
                      & MvAGC & 0.5796  & 0.3957  & 0.3224  & 0.4117  & \textbf{164.8} \\
                      & AHMvC  & \textbf{0.5904 } & \textbf{0.5298 } & \textbf{0.3729 } & \textbf{0.4892 } & 199.7 \\
                \midrule
                \multirow{2}[2]{*}{AIDS} & MvAGC & \textbf{0.5118 } & 0.0700  & 0.0182  & 0.0364  & 847.9 \\
                      & AHMvC  & 0.2629  & \textbf{0.2799 } & \textbf{0.0861 } & \textbf{0.1714 } & \textbf{752.5} \\
                \bottomrule
                \end{tabular}
                 \label{re_2}}
        	\end{table*}%
	
		 \indent GNN can capture structural information effectively, so the GNN-based methods achieve better performance than shallow methods: MNE, PMSC, and PwMC. Among GNN-based methods, CMGEC achieves the best score in most metrics. However, HMvC improves CMGEC by almost 11$\%$, 5$\%$, 14$\%$ on ACC, NMI, ARI on IMDB, respectively. This is because CMGEC doesn't explore high-order information, which will cause information loss. Compared to MvAGC, HMvC improves ACC by almost 2.1$\%$, 0.7$\%$, 3.2$\%$ on ACM, DBLP, IMDB, respectively. HMvC produces a little better performance than AHMvC. With respect to contrastive learning methods, our improvement is also significant. This indicates that capturing high-order relation in attributes is promising.   \\
		 \indent Our method is also efficient and consumes less memory compared to deep learning methods, which is appealing in practice. Specifically, benefiting from anchor idea, on medium-size attributed graph datasets, like Amazon and AIDS datasets, MvAGC and AHMvC are the most efficient methods. However, AHMvC improves MvAGC on ACC by more than 5$\%$ and 1$\%$ on Amazon photos, Amazon computers, respectively. On AIDS, MvAGC achieves better ACC while AHMvC is much better on all other metrics. In fact, 
		high-order information is limited on AIDS since its feature just have 4 dimensions. From another perspective, our optimization converges fast. As observed from Fig. \ref{losses}, the value of HMvC/AHMvC's objective reaches its lower bound in less than 20 iterations. 
        	    \begin{figure*}[hbtp]
        	    
    		    \centering
    			\includegraphics[width=0.32\textwidth]{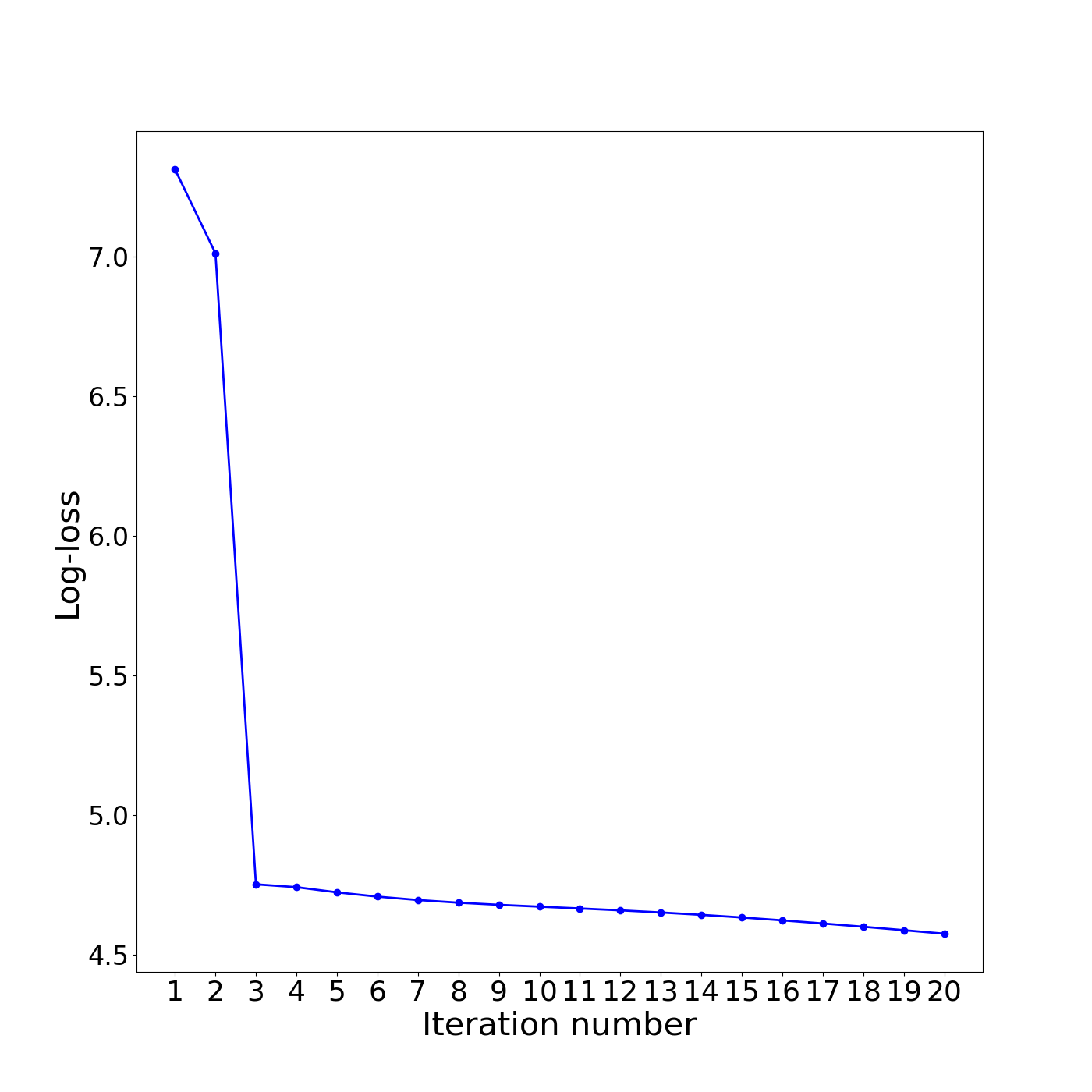}
    			\includegraphics[width=0.32\textwidth]{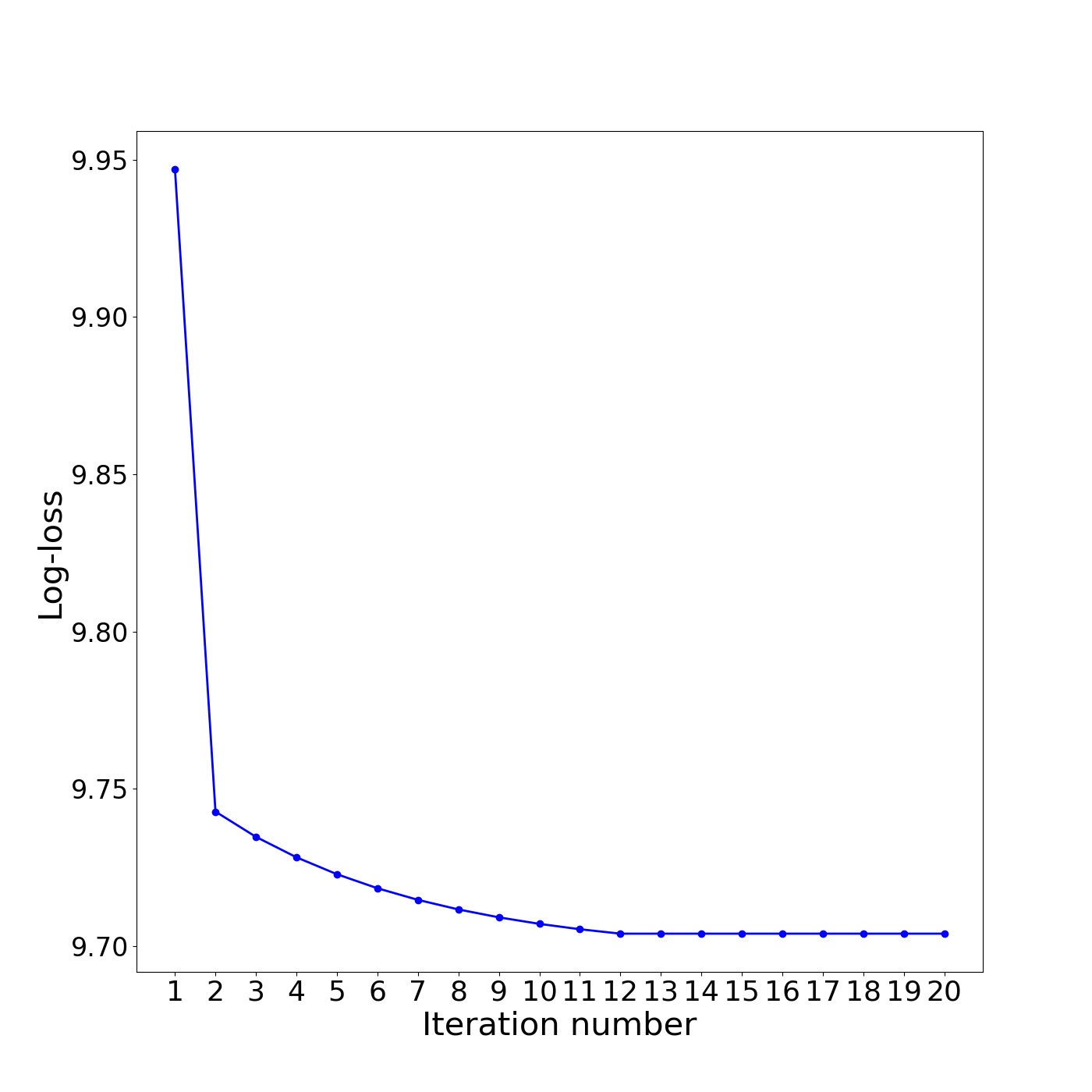}
    			\includegraphics[width=0.32\textwidth]{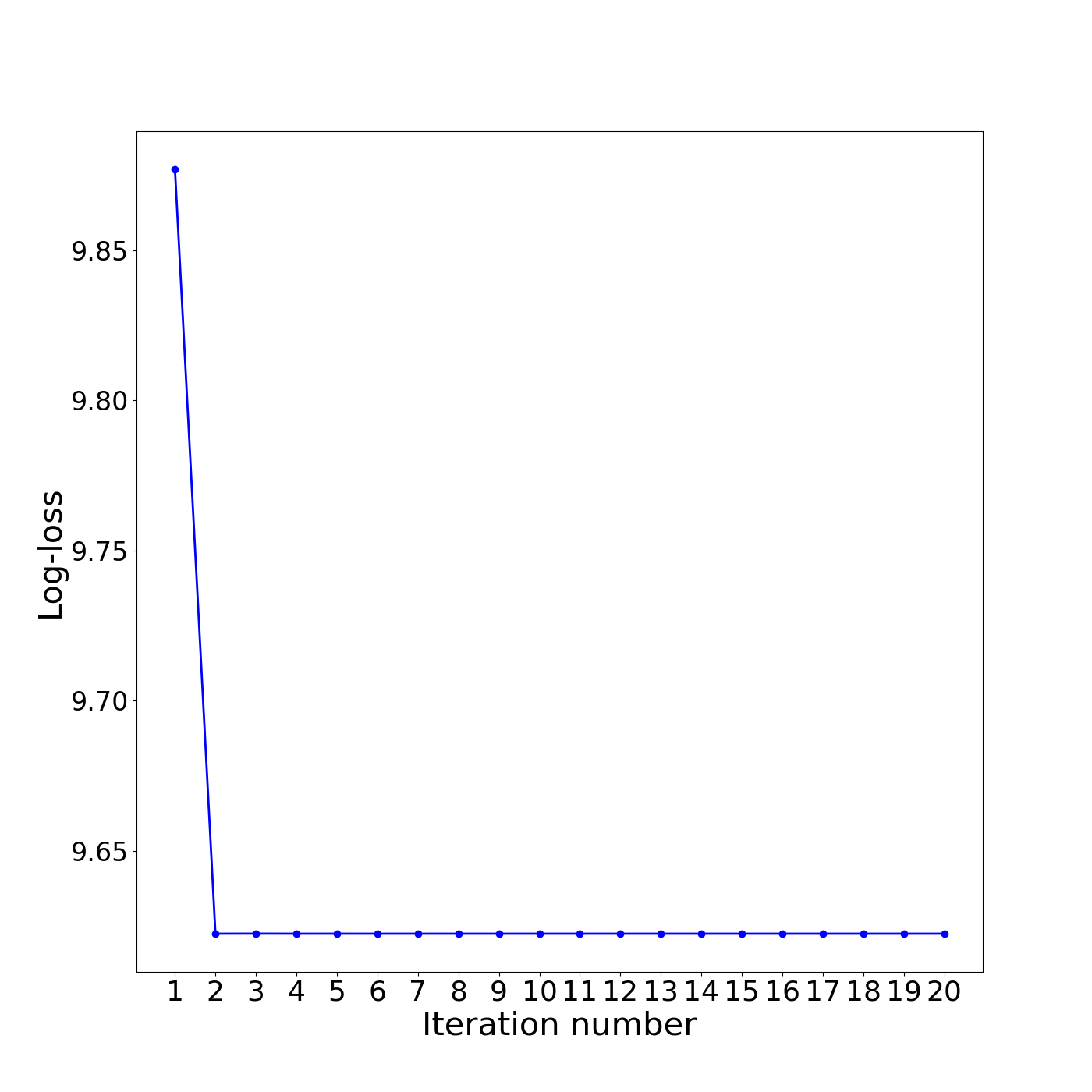}
    			\caption{The objective variation on ACM, Amazon photos and Amazon computers.}
    			\label{losses}
    		\end{figure*}

		\subsubsection{Clustering task on non-graph data}

		All results on feature data are shown in Table \ref{re_3}. 
		\begin{table*}[htbp]
		 \renewcommand\arraystretch{0.8}
			\centering
            \small 
			\caption{The results on feature data ($k=2, n=2$).}
    \begin{tabular}{cccccccc}
    \toprule
    Datasets & Metrics & AMSL  & LMVSC & MLRSSC & MSC\_IAS & SMC   & HMVC \\
    \midrule
    \multirow{3}[1]{*}{Caltech7} & ACC   & 0.4518  & 0.7266  & 0.3731  & 0.3976  & 0.7869  & \textbf{0.8215 } \\
          & NMI   & 0.4243  & 0.5193  & 0.2111  & 0.2455  & 0.4829  & \textbf{0.5696 } \\
          & PUR   & 0.4674  & 0.7517  & 0.4145  & 0.4444  & \textbf{0.8860 } & 0.8398  \\
          \midrule
    \multirow{3}[0]{*}{Caltech20} & ACC   & 0.3013  & 0.5306  & 0.3731  & 0.3127  & 0.5716  & \textbf{0.6249 } \\
          & NMI   & 0.4054  & 0.5271  & 0.2111  & 0.3138  & 0.5458  & \textbf{0.6327 } \\
          & PUR   & 0.3164  & 0.5847  & 0.4145  & 0.3374  & 0.6266  & \textbf{0.7225 } \\
          \midrule
    \multirow{3}[0]{*}{Citeseer} & ACC   & 0.1687  & 0.5226  & 0.2509  & 0.3411  & 0.5600  & \textbf{0.6419 } \\
          & NMI   & 0.0023  & 0.2571  & 0.0267  & 0.1153  & 0.2985  & \textbf{0.3707 } \\
          & PUR   & 0.1687  & 0.5446  & 0.6370  & \textbf{0.8076 } & 0.5600  & 0.6542  \\
          \bottomrule
    \end{tabular}%
			\label{re_3}%
		\end{table*}%
	
		\indent Although SMC is the best one among compared methods on these datasets, HMvC outperforms SMC by almost 4$\%$, 5$\%$, 8$\%$ on ACC on three benchmarks. The improvements of NMI and PUR are up to 8$\%$ and 5$\%$ on average. This is mainly attributed to the exploitation of high-order information. These results indicate that our approach can handle both graph and generic feature data effectively.
	
		\section{Ablation Studies}
	
		\subsection{Effect of Graph Filtering}
         \begin{figure*}[t]
    	    \begin{minipage}{0.48\textwidth}
            \includegraphics[width=1\textwidth]{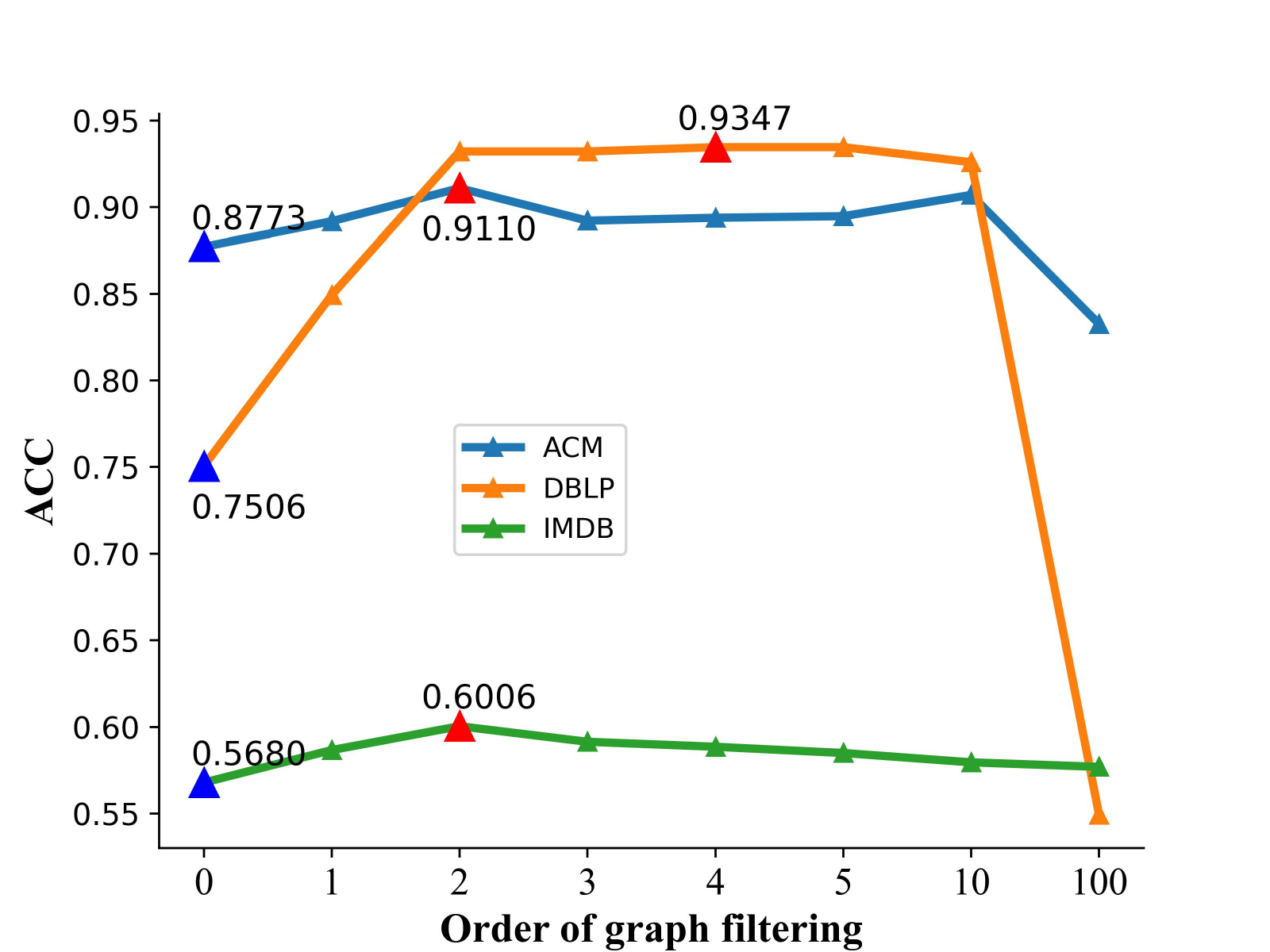}
            \caption{Effect of graph filtering order $k$. }
          \label{efk}
    	    \end{minipage}
    	     \begin{minipage}{0.48\textwidth}
    	    \includegraphics[width=1\textwidth]{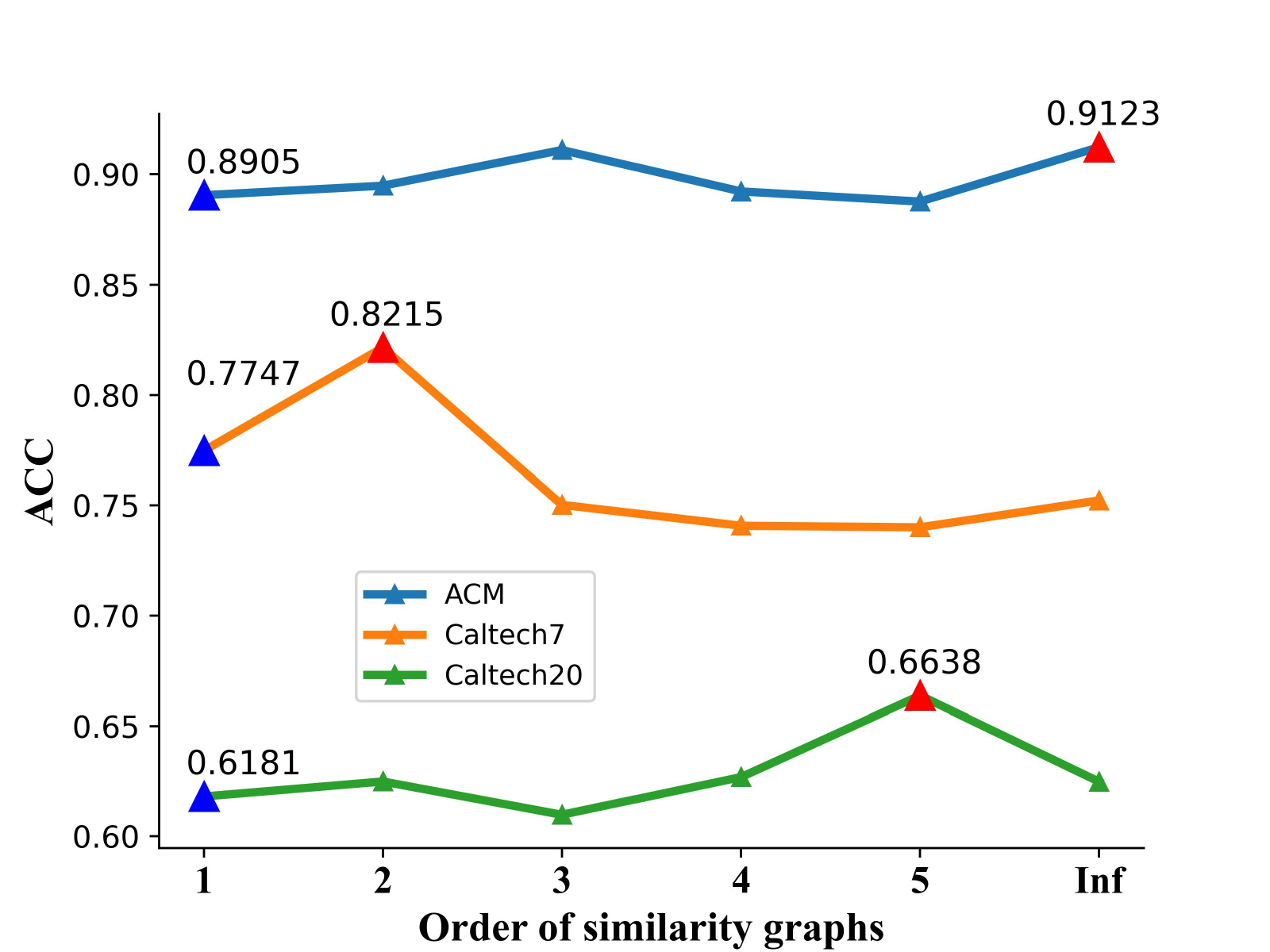} 
    	    \caption{Effect of  similarity graphs' order $n$. }\label{efn}
    	    \end{minipage}
    	   \end{figure*}

          
		\indent Graph filtering is aimed to filter out undesirable high-frequency noise as well as retaining graph geometric features \cite{lin2021graph}. To show its effect in HMvC, we report the results of different $k$ on three multi-view graph datasets in Fig. \ref{efk}. \textcolor{red}{Red} Triangles indicate the best results, while \textcolor{blue}{bule} ones represent results without graph filtering. As $k$ increases, all results get better first but then become worse, and the best results are obtained around $k= 2$. Too large $k$ are not necessary since it will lead to over-smoothing problem, where the filtered representation becomes too smooth to distinguish.\\
	     \indent We can see that the effects on different datasets are different. The improvement on DBLP is more than 18$\%$, while 4$\%$ and 3$\%$ on ACM and IMDB. This is because distribution of  eigenvalues varies on different datasets \cite{magc4}. Specifically, there are more high-frequency components in IMDB than others, which would weaken the effectiveness of graph filter.
    				
		\subsection{Effect of high-order neighborhood information} \label{ech}
       \indent We set different order $n$ to explore the effect of high-order relationships. In Fig. \ref{efn}, \textcolor{blue}{blue} triangles indicate the results with the first-order graph. Caltech7 and Caltech20 achieve the best performance at $n=2$ and $n=5$, and ACC improvement is almost 5$\%$ with respect to $n=1$. Therefore, the results in Tables \ref{re_1}-\ref{re_3} can be further improved if $n$ is tuned. On ACM, HMvC obtains better results on the third-order similarity graph, and the improvement of ACC is up to 2.1$\%$ compared with the first-order. It also benefit from the infinity-order graph. Although exploring high-order neighborhood information with a larger $n$ may not achieve better performance because there could be too much extraneous noisy information, the result of high-order similarity is often better than that of the first-order one. The suitable value of $n$ is chosen less than 5 or the infinite value empirically.\\
      \begin{figure*}[htbp]
      \centering
                	\subfloat[Citeseer]{
                		\includegraphics[width=0.23\textwidth]{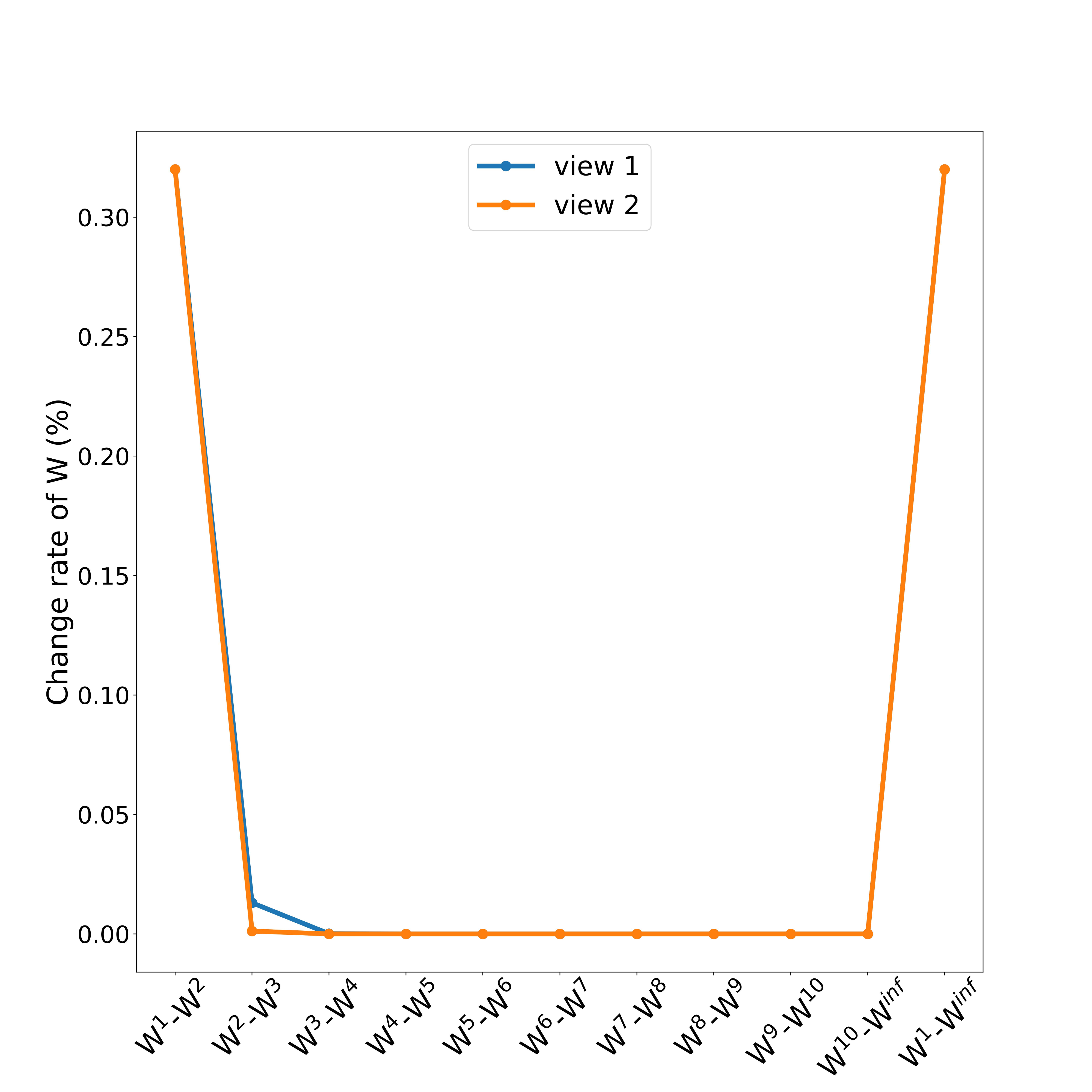}
                	}	\hspace{-4pt}
                	\subfloat[DBLP]{
                		\includegraphics[width=0.23\textwidth]{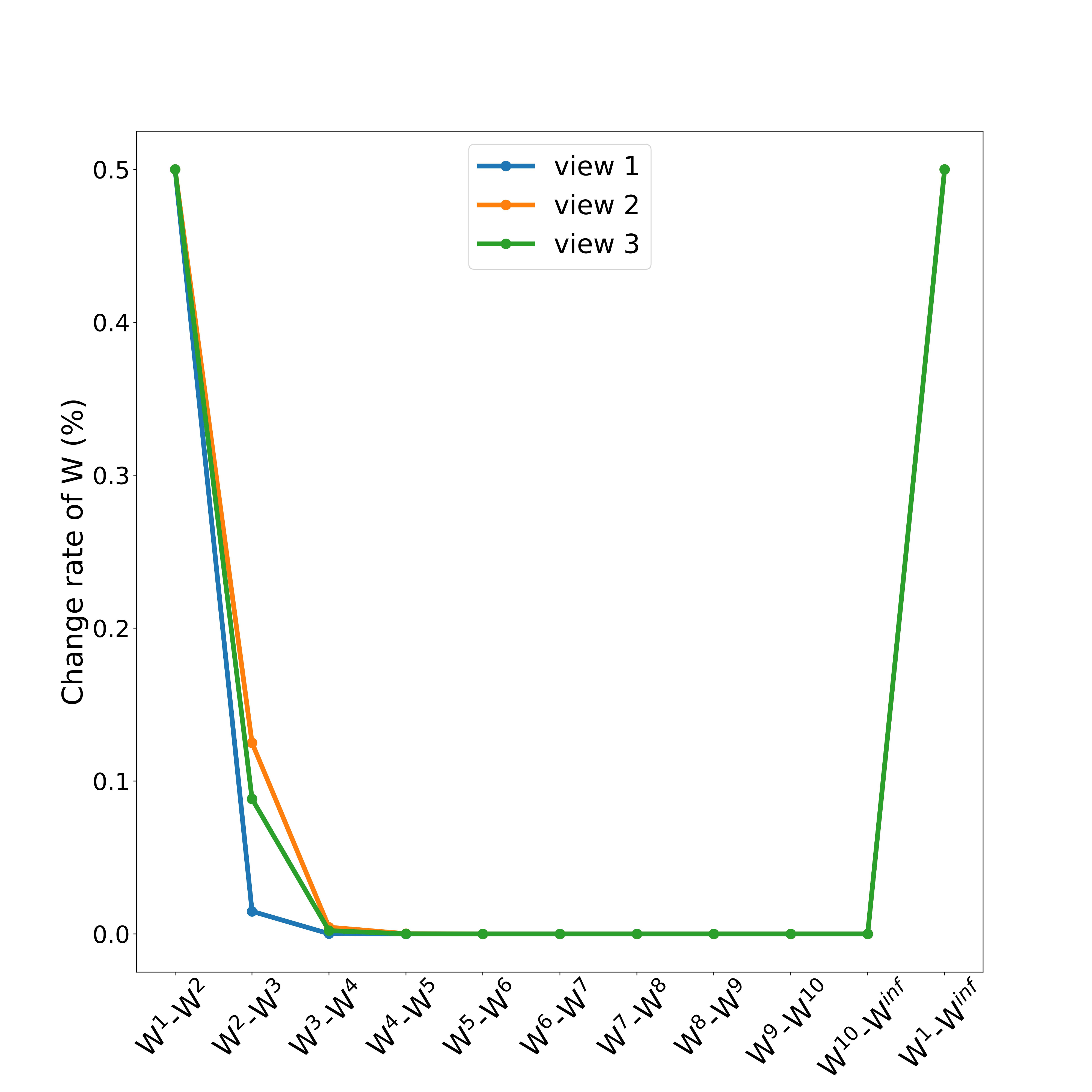}
                	}
                	\subfloat[Caltech7]{
                		\includegraphics[width=0.23\textwidth]{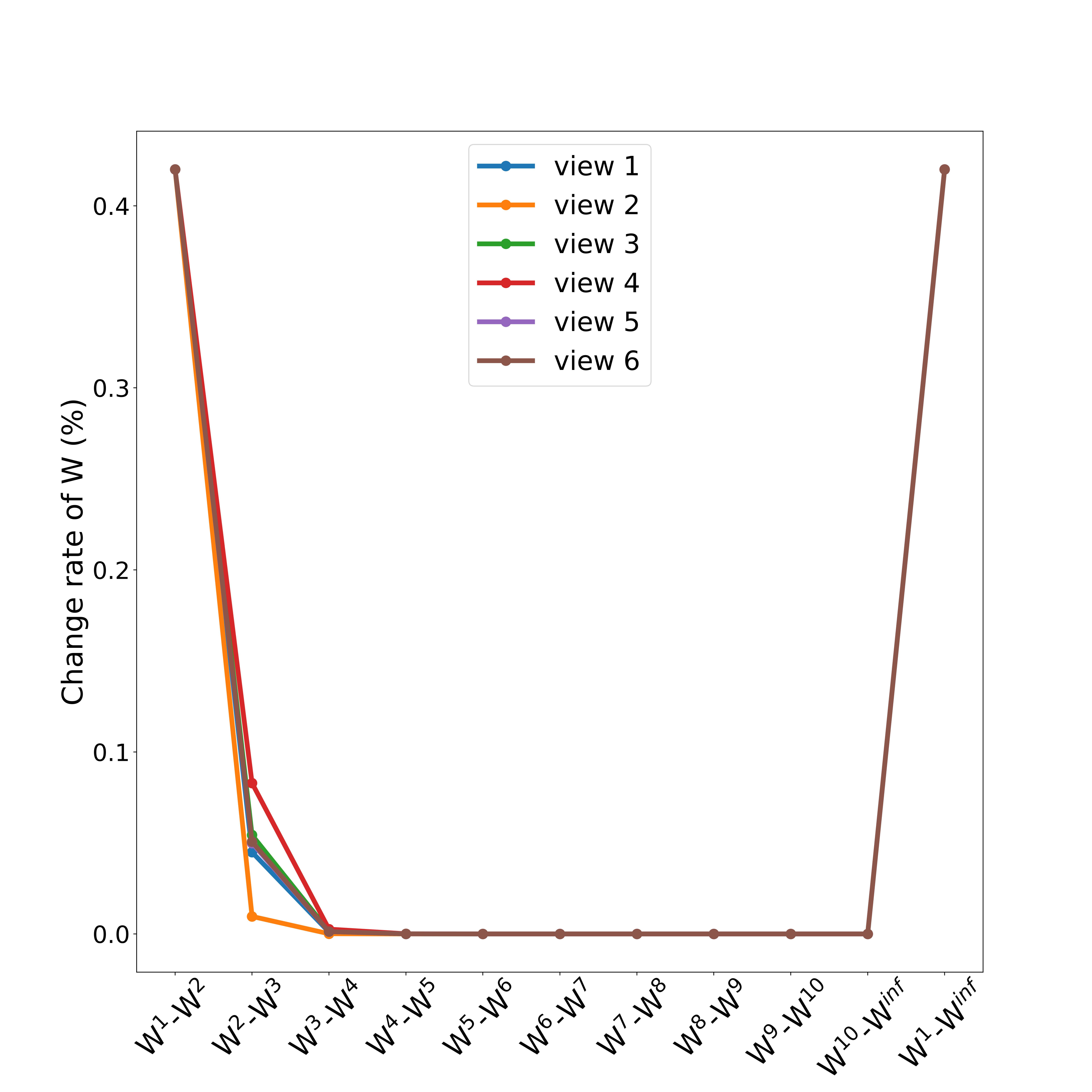}
                	}	\hspace{-4pt}
                	\subfloat[Caltech20]{
                		\includegraphics[width=0.23\textwidth]{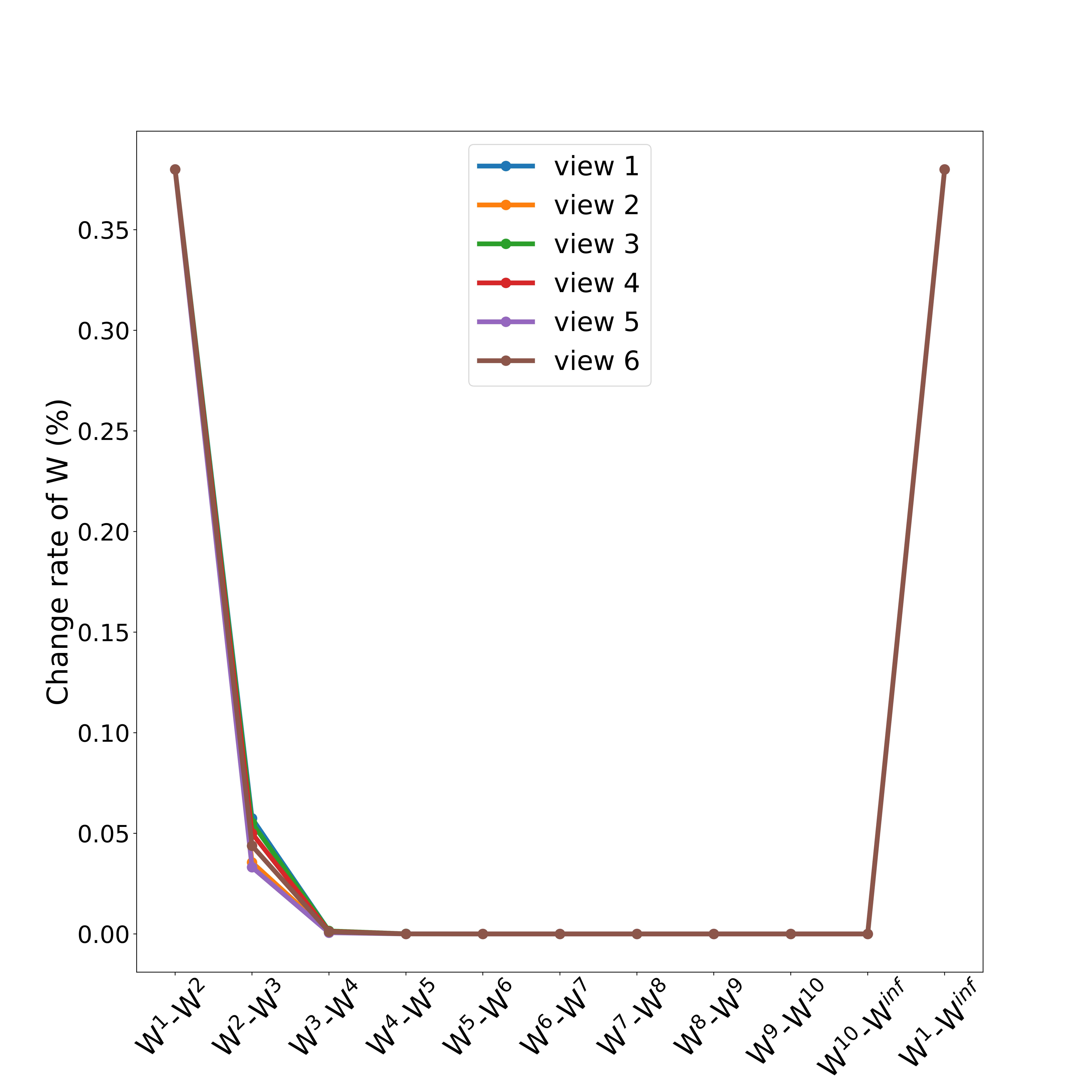}
                	}
                	\caption{Rate of change in similarity matrices between different orders. }
                	\label{rate}
    	    \end{figure*}
       \indent In Fig. \ref{rate}, we demonstrate the average of change rate, i.e., $\frac{|{(W_v^n)}_{ij}-{(W_v^{n-1})}_{ij}|}{|{(W_v^{n-1})}_{ij}|}$,  in similarity matrices between different orders on Citeseer, DBLP, Caltech7, and Caltech20. Clearly, there exists u-shape and two big changes on all results. The first one occurs between first-order and second-order, which means the information of these two matrices are much different. The big change between the finite-order and infinity-order matrix suggests that the infinity-order similarity contain some unique information. There is little change between other orders. This analysis validates that the second-order and infinity-order similarity matrices contain much valuable information.
                
           	\subsection{Parameter Analysis} 
           	\label{paras ans}
			\indent There are three trade-off parameters in our method, $\alpha$, $\beta$, and $\mu$. Fig. \ref{paras} shows sensitivity analysis of three parameters on Caltech7. We fix one parameter and vary others in each figure. We find too large or too small $\alpha$ results in sub-optimal performance. When fix $\alpha$ at a suitable value, we find results change little if $\beta$ or $\mu$ varies. However, too large $\beta$ or $\mu$ weakens high-order graph part  and results in inferior performance.
    		\begin{figure*}[hbtp]
    			\centering
    				\includegraphics[width=0.23\textwidth]{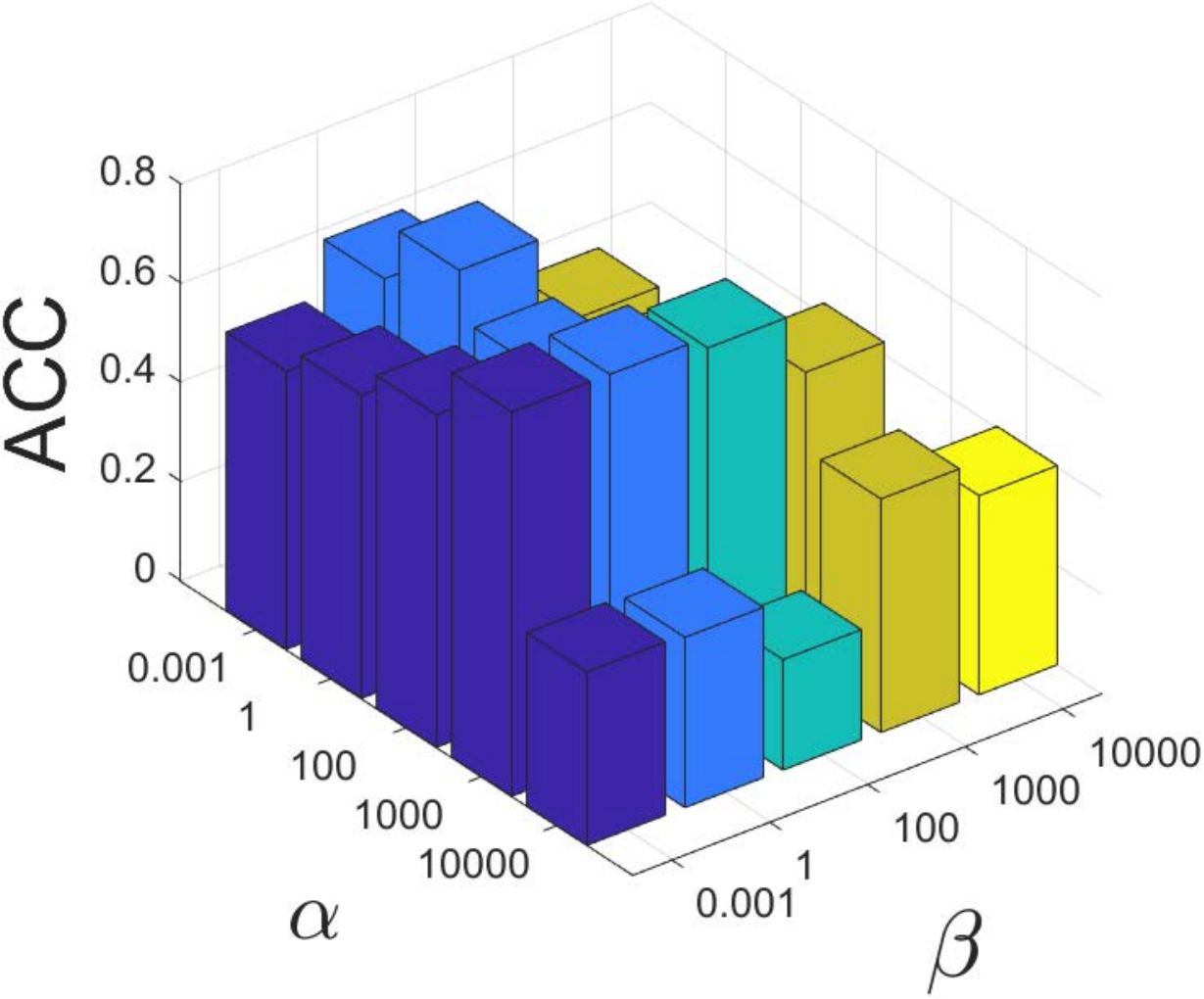}
    				\includegraphics[width=0.23\textwidth]{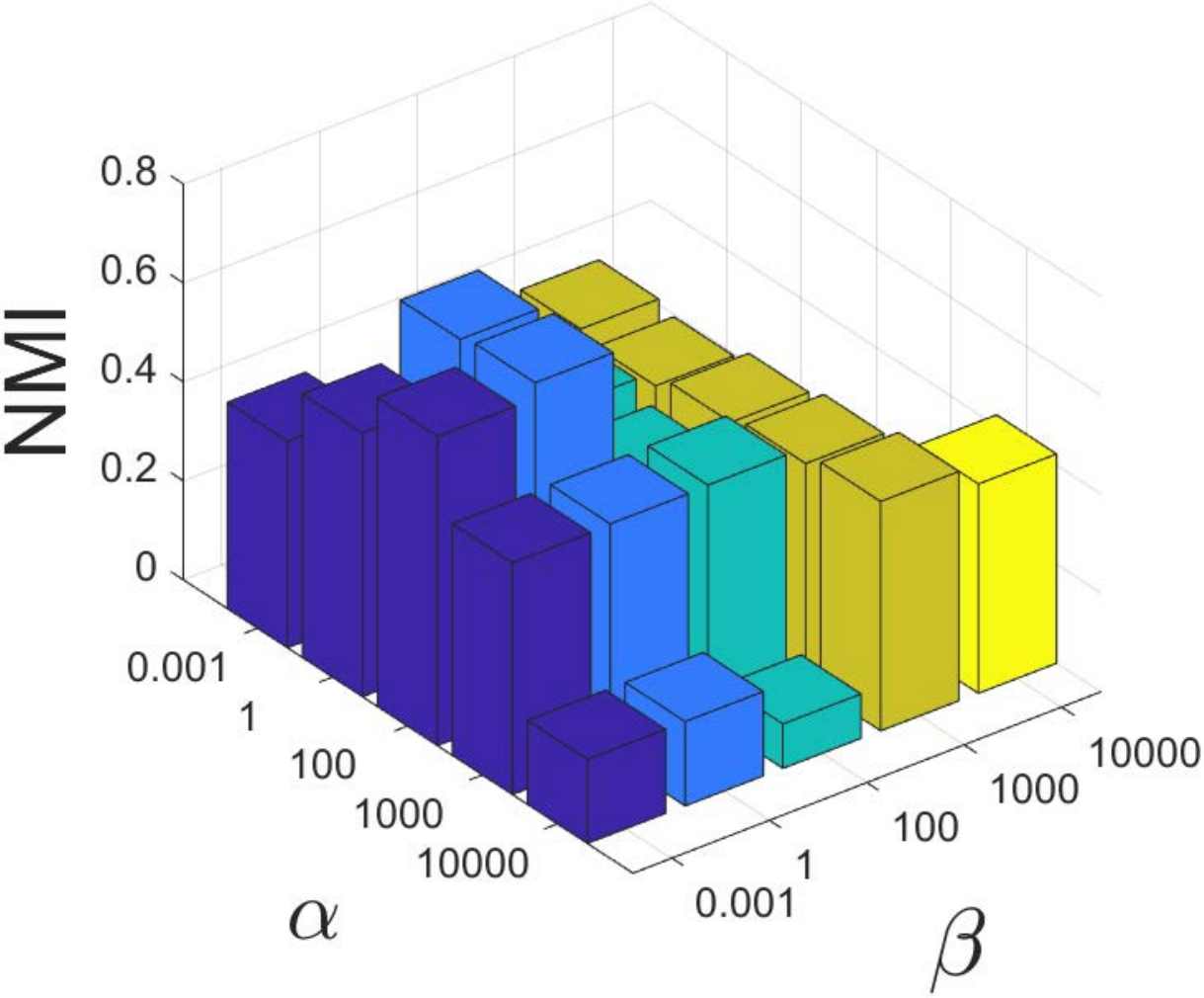}
    				\includegraphics[width=0.23\textwidth]{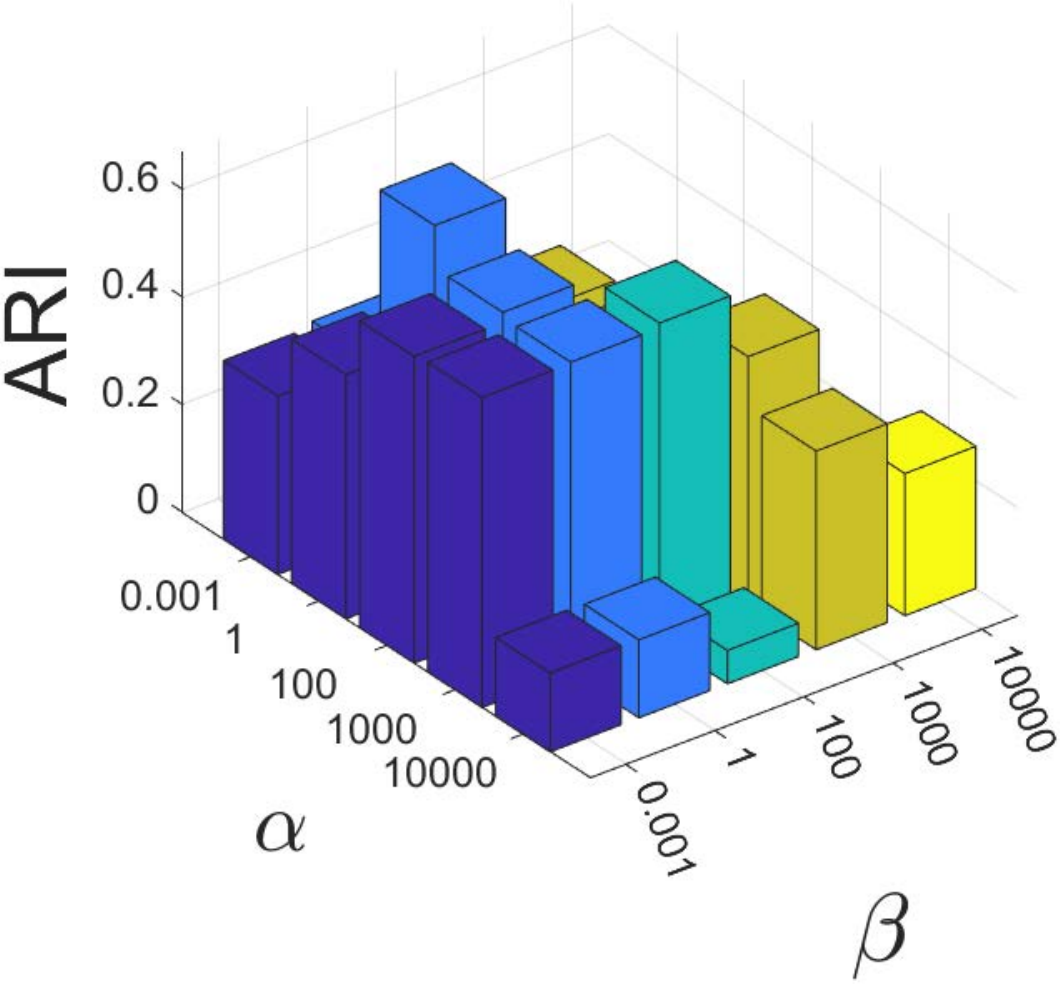}
    				\includegraphics[width=0.23\textwidth]{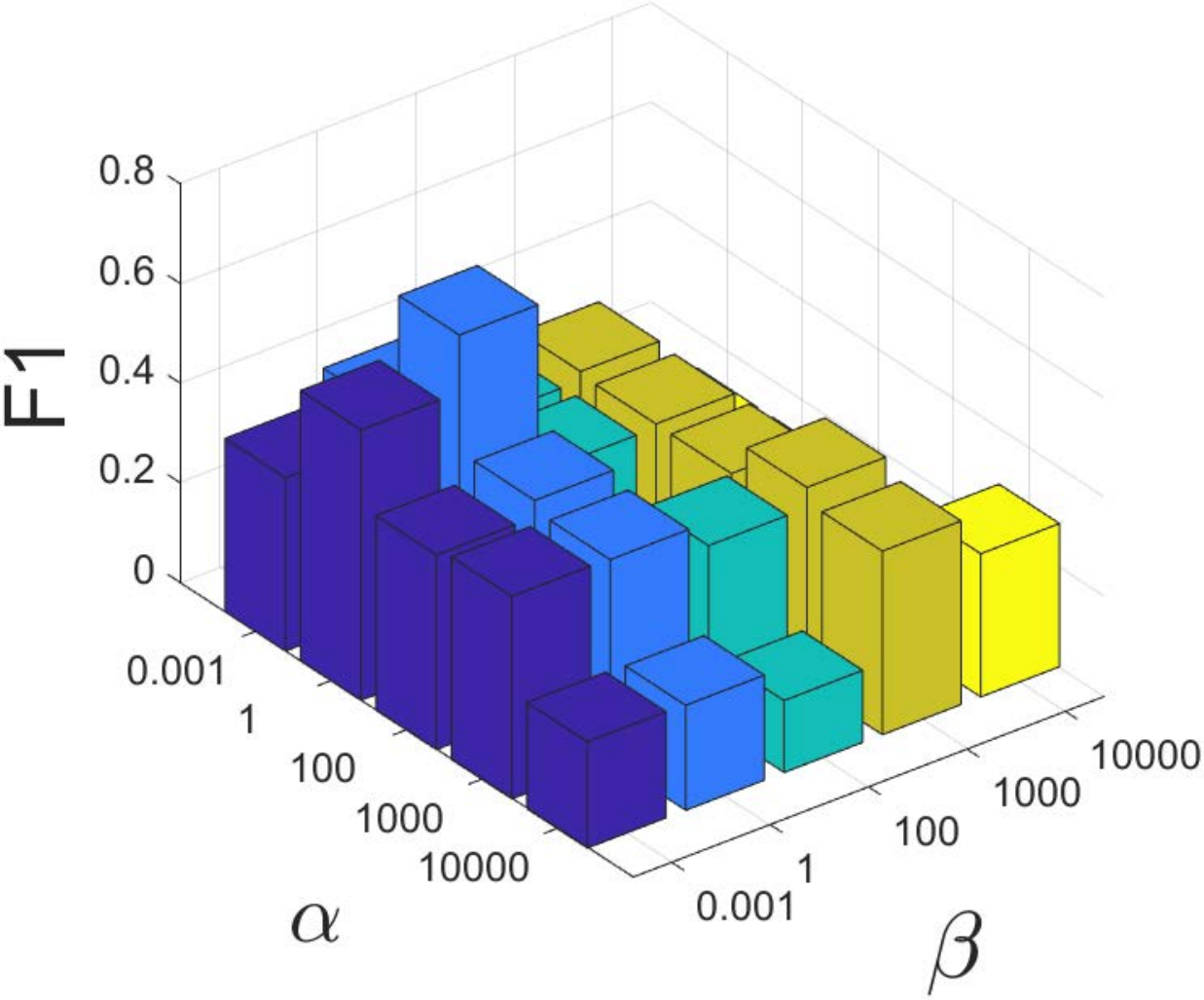}
    				\includegraphics[width=0.23\textwidth]{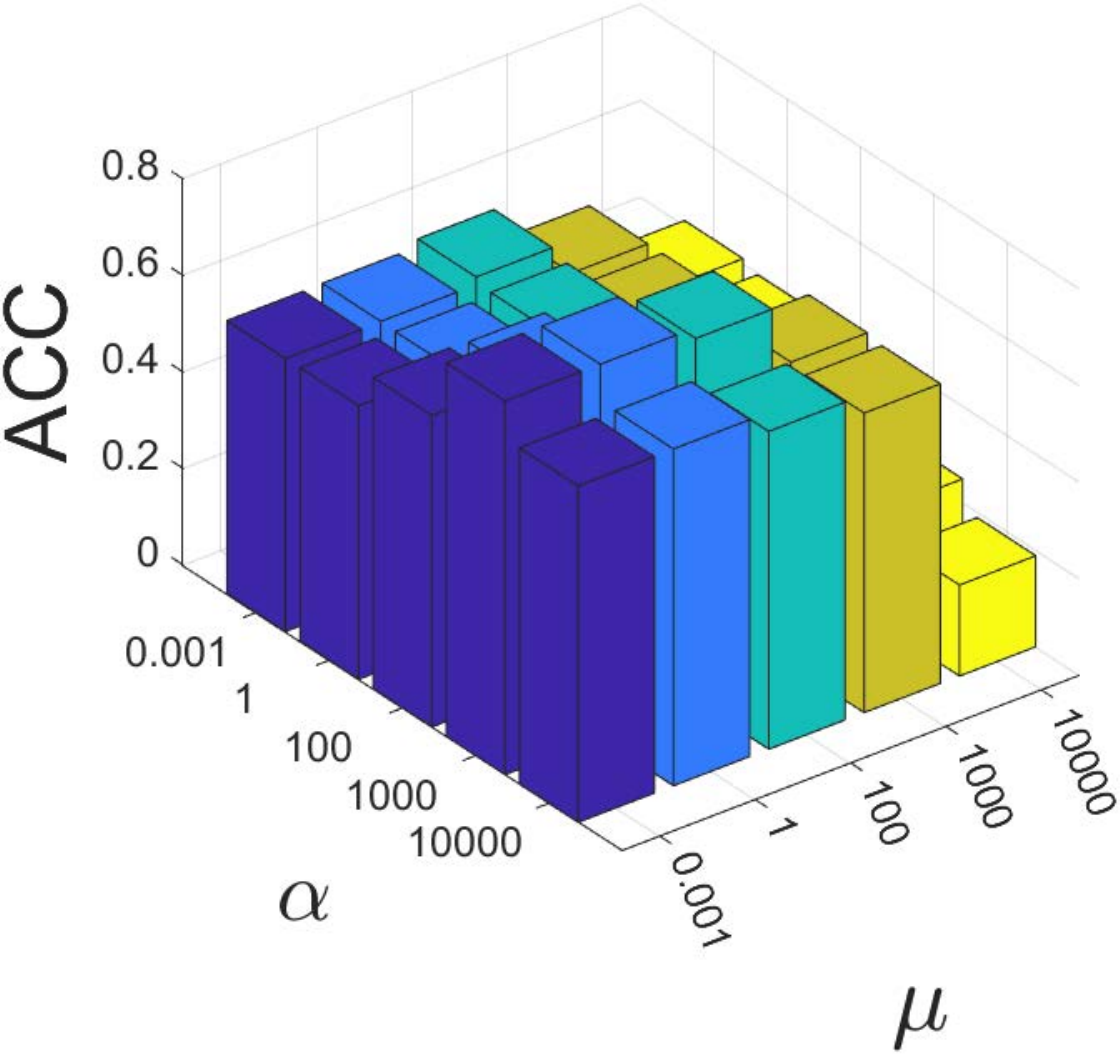}
    				\includegraphics[width=0.23\textwidth]{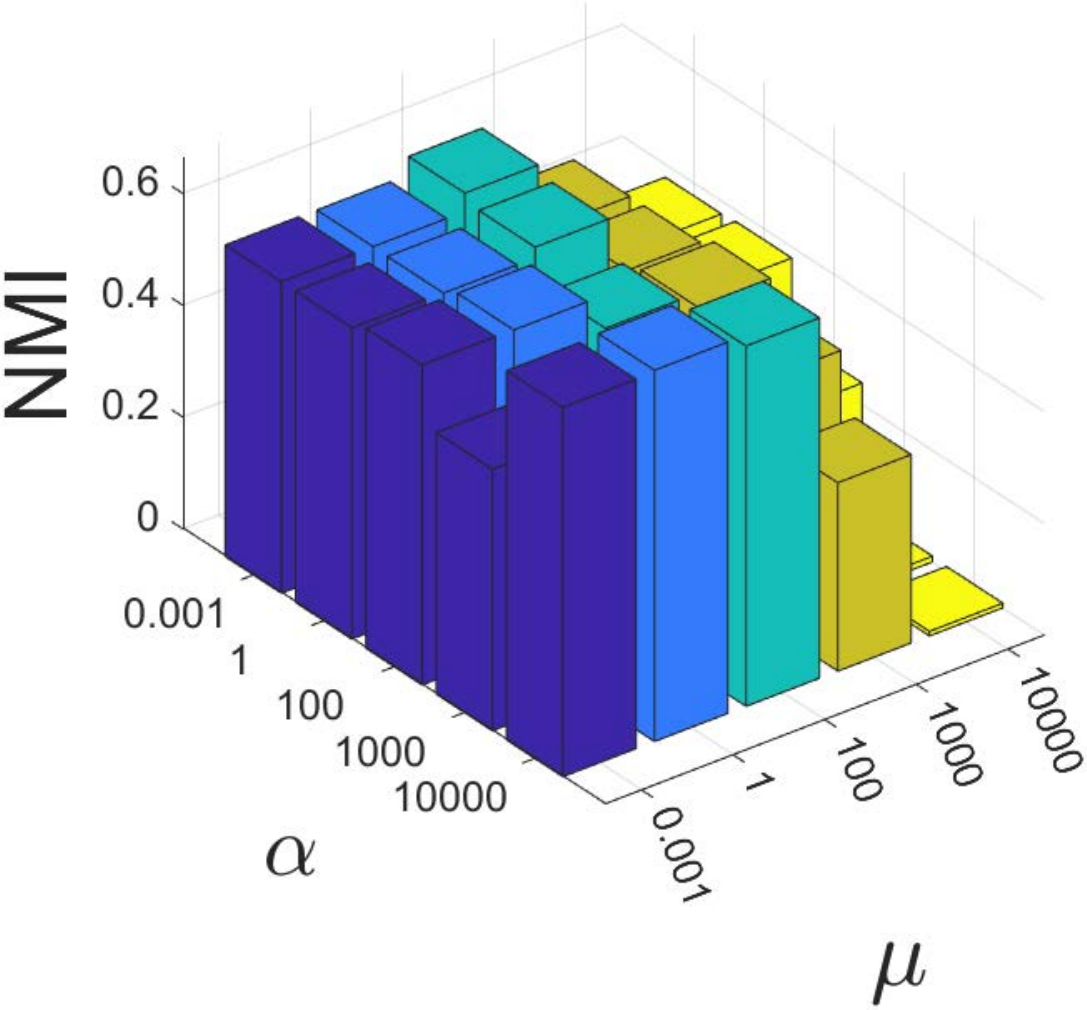}
    				\includegraphics[width=0.23\textwidth]{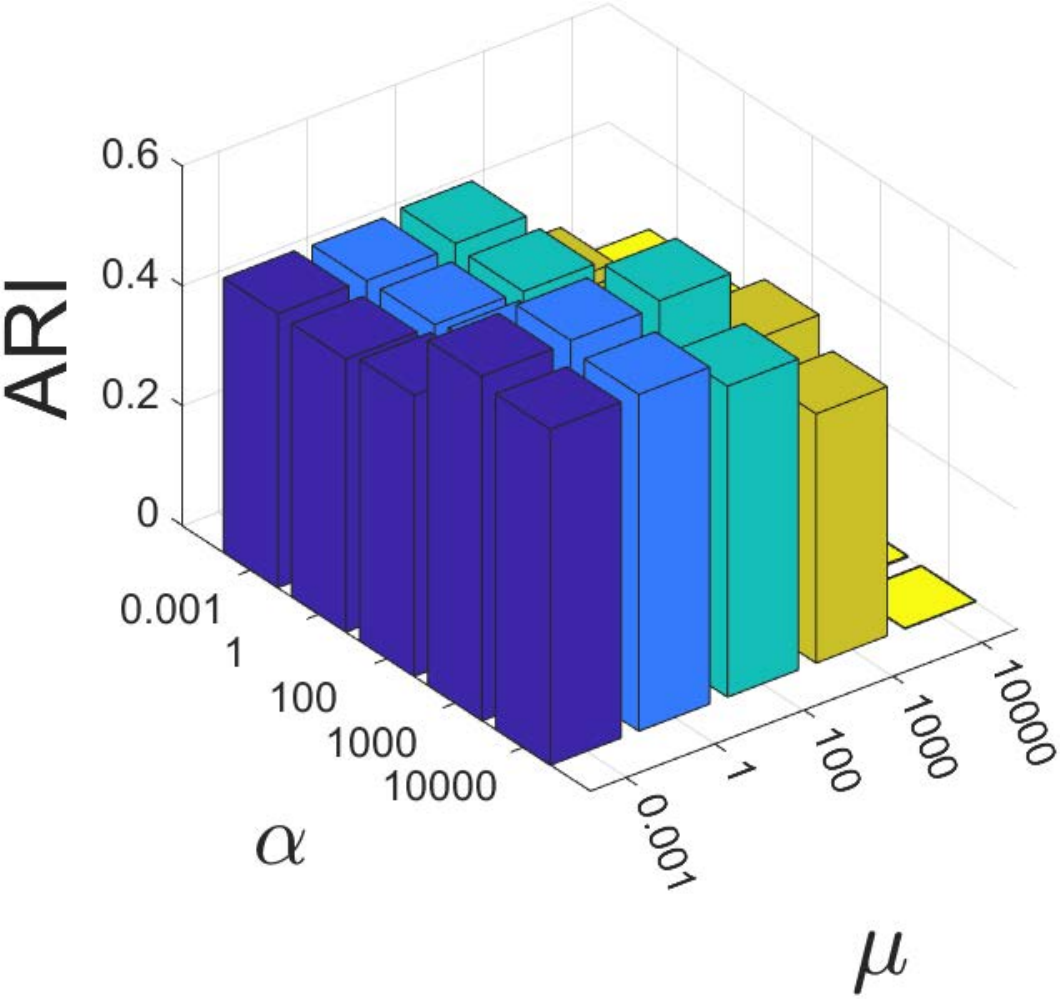}
    				\includegraphics[width=0.23\textwidth]{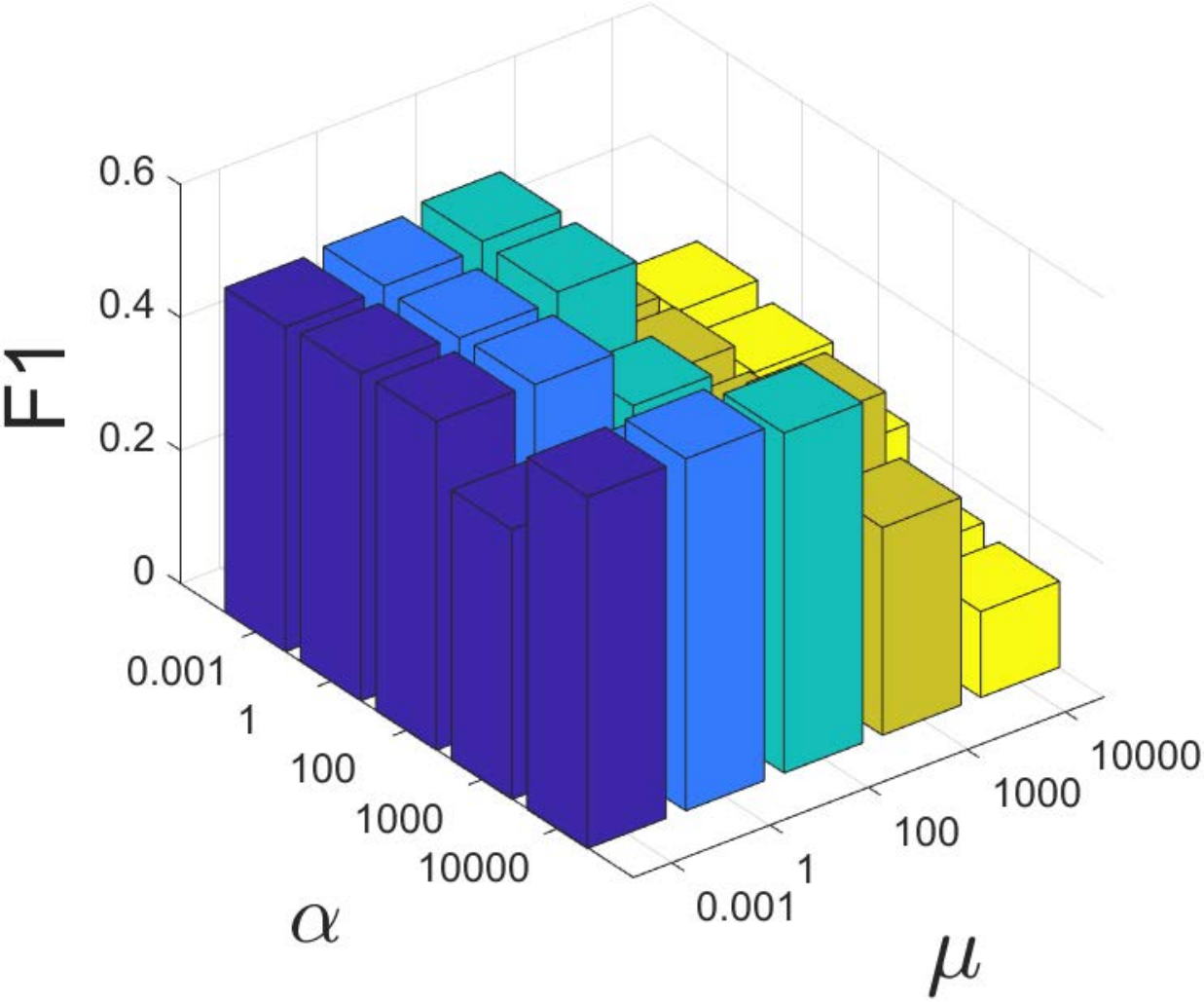}
    				\includegraphics[width=0.23\textwidth]{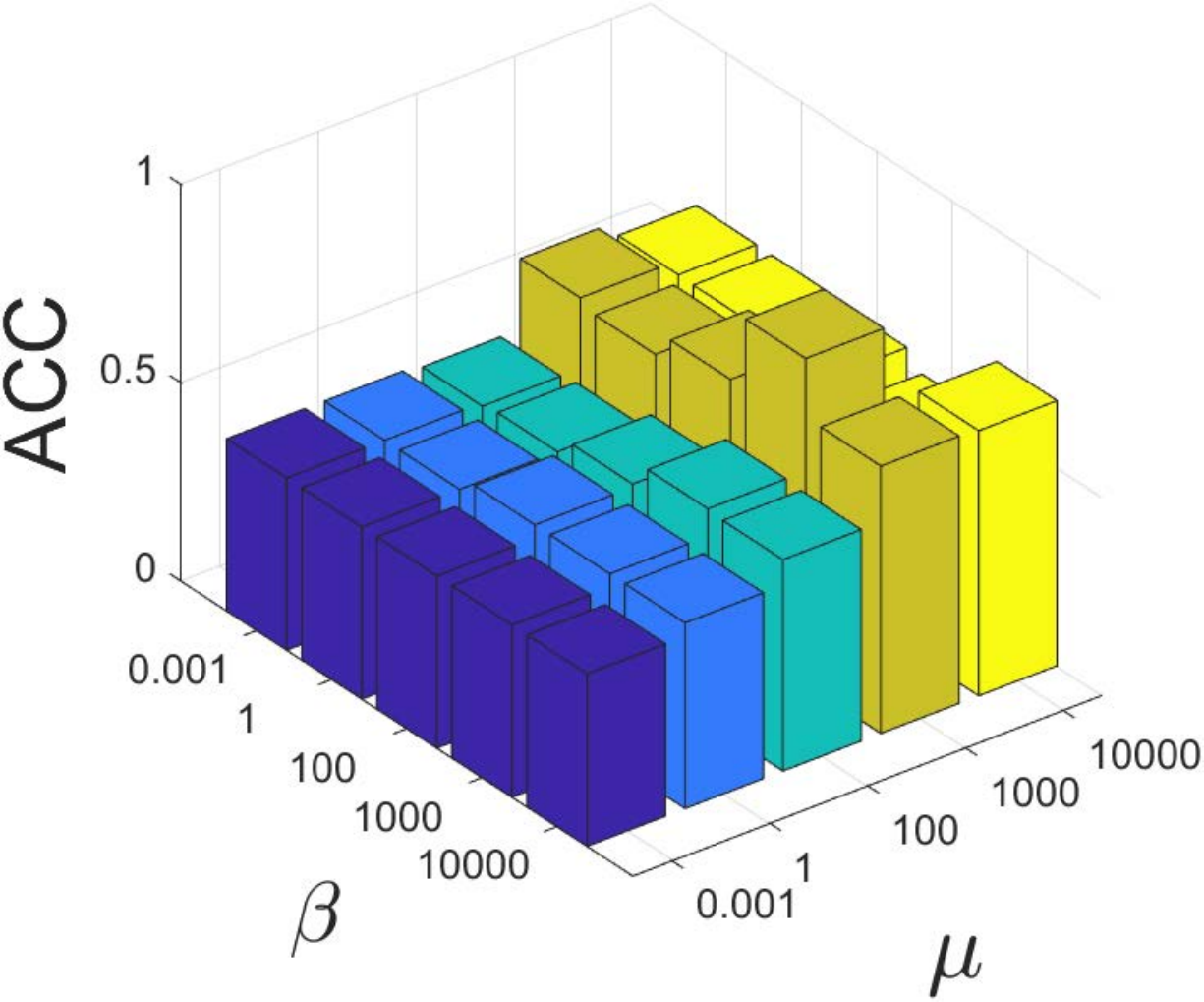}
    				\includegraphics[width=0.23\textwidth]{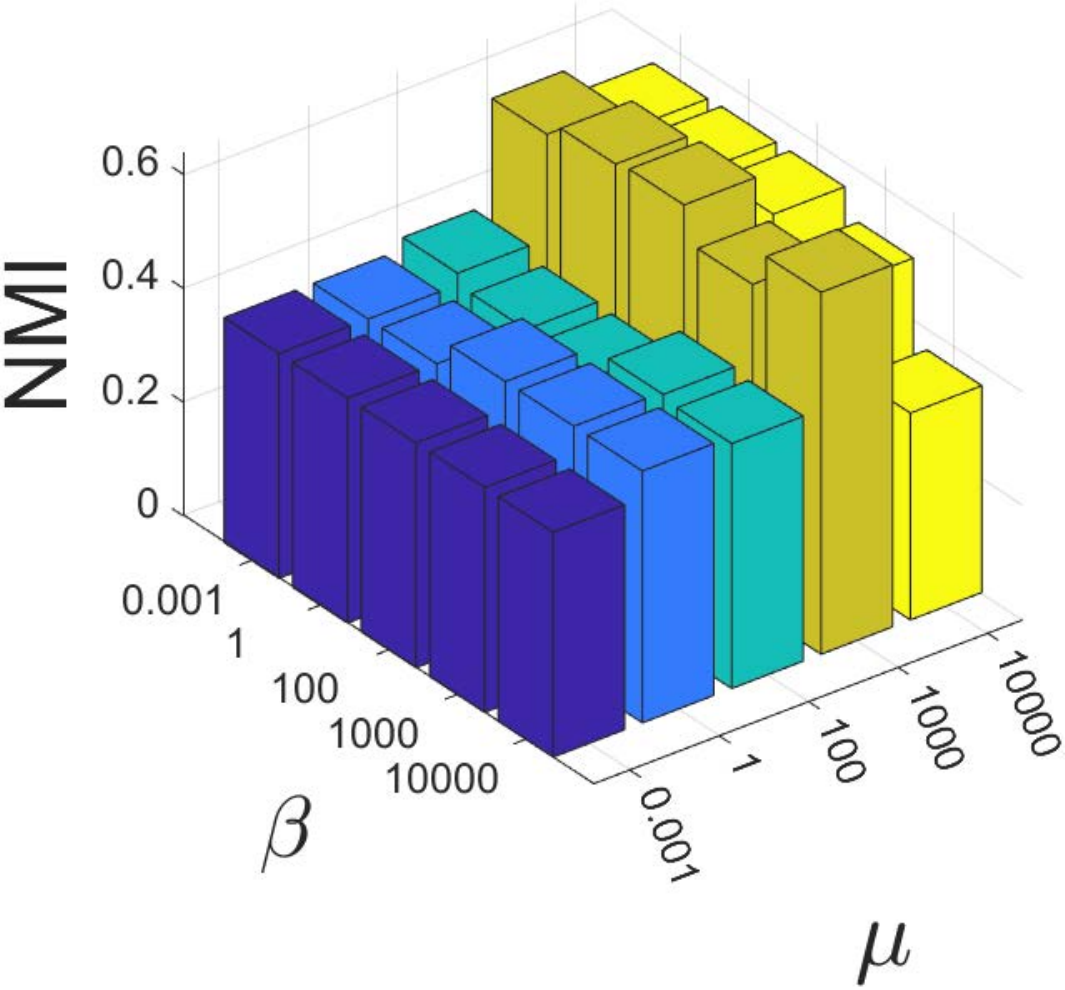}
    				\includegraphics[width=0.23\textwidth]{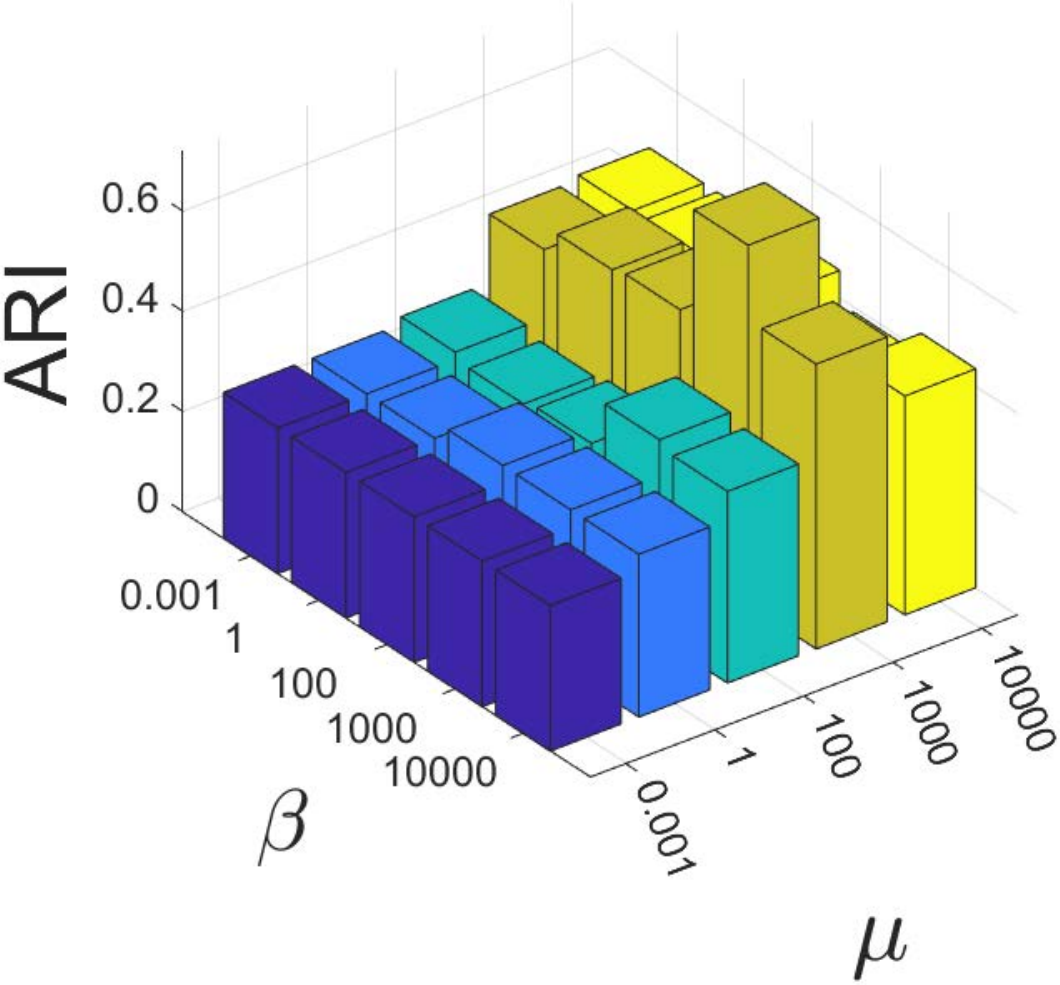}
    				\includegraphics[width=0.23\textwidth]{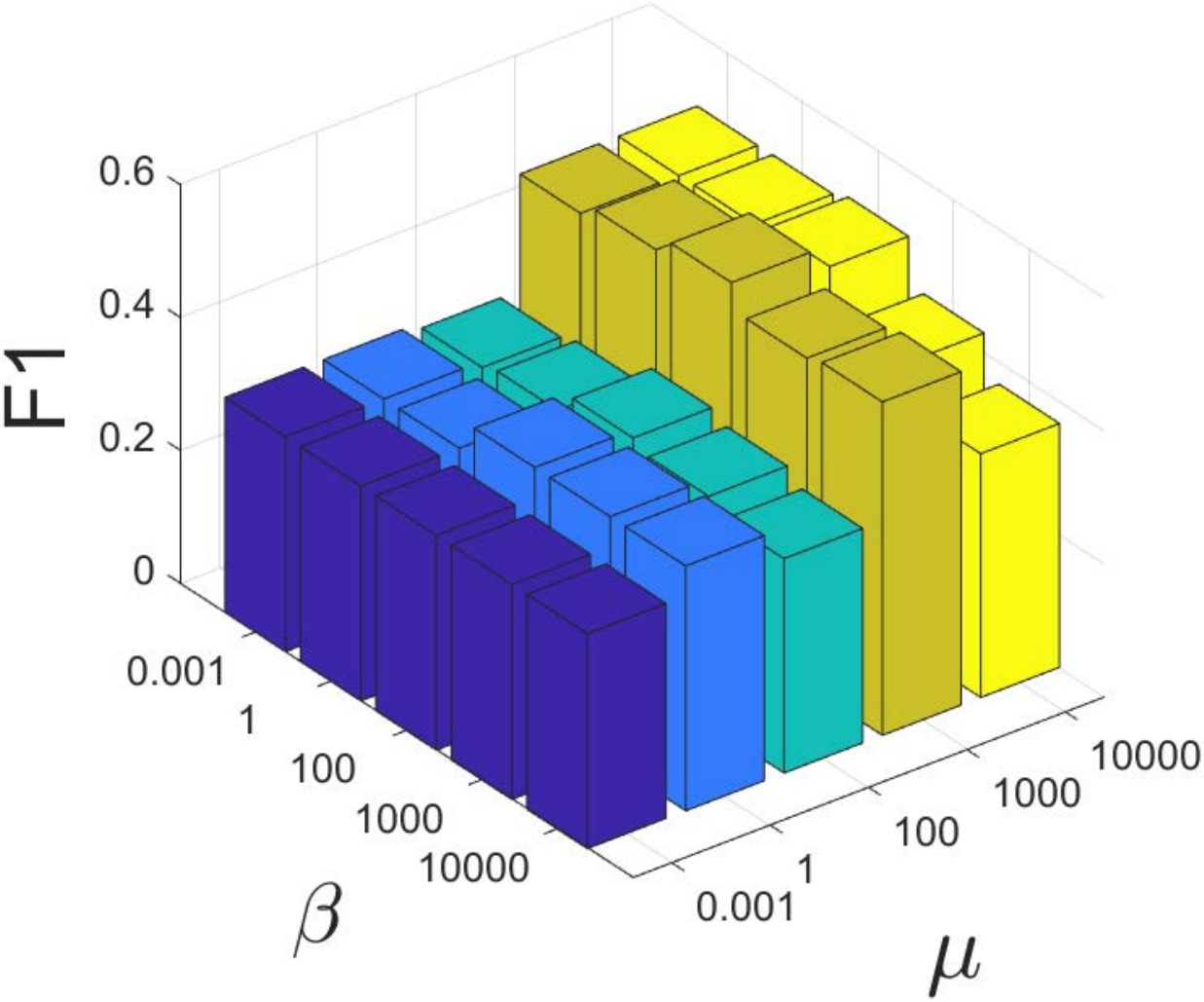}
    				\caption{Sensitivity analysis of parameters for our method on Caltech7 with different $\alpha$, $\beta$, and $\mu$}
    				\label{paras}
    			\end{figure*}
            
			\section{Conclusion}
			\indent In this paper, we develop a generic multi-view clustering method, which effectively explores the high-order information in both topology and feature space. To obtain a smooth representation and unify the process of graph and non-graph data, we apply graph filtering technique to handle the raw features. To capture high-order feature information, we define the high-order and infinity-order graphs by leveraging the powers of similarity matrix and convergent matrix. To solve the scalability issue on medium-size or large-scale dataset, we use the anchor idea to construct scalable high-order graphs. Our approach achieves promising performance on generic data. Therefore, it is beneficial to treat non-graph data from the perspective of graph. In addition, the proposed infinity graph is general to be broadly applied in many other machine learning tasks.  However, one potential limitation of our work is that the graph construction consumes time and space. In future work, we plan to improve the graph filter and construction strategy of one-time similarity graph.

			
			
        \bibliographystyle{IEEEtran}
        \bibliography{reference}
\begin{IEEEbiographynophoto}{Erlin Pan} received the B.Sc.
degree in computer science and technology in 2021 and is currently
studying for a M.S degree with School of Computer Science and Engineering, the University of Electronic Science and Technology of China,  Chengdu, China.His main research interests include graph leanring,
multi-view learning, and clustering.
\end{IEEEbiographynophoto}

\begin{IEEEbiographynophoto}{Zhao Kang} received the Ph.D. degree in computer science from Southern Illinois University Carbondale, Carbondale, IL, USA, in 2017. He is currently an Associate Professor with the School of Computer Science and Engineering, University of Electronic Science and Technology of China, Chengdu, China. He has published over 80 research papers in top-tier conferences and journals, including \textit{NeurIPS}, \textit{AAAI}, \textit{IJCAI}, \textit{\textit{ICDE}, \textit{CVPR}, \textit{SIGKDD}, \textit{ECCV}, \textit{ICDM}, \textit{CIKM}, \textit{SDM}, \textit{IEEE TRANSACTIONS ON CYBERNETICS, IEEE Transactions on Image Processing}, \textit{IEEE Transactions on Knowledge and Data Engineering}, and \textit{IEEE Transactions on Neural Networks and Learning Systems}. His research interests are machine learning, pattern recognition, and data mining. Dr. Kang has been an AC/SPC/PC Member or a Reviewer for a number of top conferences, such as \textit{NeurIPS}, \textit{ICLR}, \textit{AAAI}, \textit{IJCAI}, \textit{CVPR}, \textit{ICCV}, \textit{MM}, and \textit{ECCV}. He regularly serves as a Reviewer for \textit{the Journal of Machine Learning Research}, \textit{IEEE TRANSACTIONS ON PATTERN ANALYSIS AND MACHINE INTELLIGENCE}, \textit{IEEE TRANSACTIONS ON NEURAL NETWORKS AND LEARNING SYSTEMS}, \textit{IEEE TRANSACTIONS ON CYBERNETICS}, IEEE TRANSACTIONS ON KNOWLEDGE AND DATA ENGINEERING}, and \textit{IEEE TRANSACTIONS ON MULTIMEDIA}.
\end{IEEEbiographynophoto}

\end{document}